\documentclass{article}
\usepackage{microtype}
\usepackage{graphicx}
\usepackage{subcaption}
\usepackage{booktabs} % for professional tables

\usepackage{hyperref}

\usepackage[accepted]{icml2020}

\icmltitlerunning{Stochastic Hamiltonian Gradient Methods for Smooth Games}
\usepackage{amsmath,amsfonts,amssymb,amsthm,array}
\usepackage{algorithm}
\usepackage{algorithmic}%
\usepackage{bm}
\usepackage{color}
\usepackage{graphicx}
\usepackage{enumitem}

\newcommand{\Exp}{\mathbb{E}}

\newcommand{\E}[1]{{\mathbb{E}\left[#1\right] }}    % expectation
\newcommand{\EE}[2]{{\mathbb{E}_{#1}\left[#2\right] }} 

\newcommand{\Prob}[1]{\mathbb{P} \left[ #1\right]}
\newcommand{\R}{\mathbb{R}}

\usepackage{multirow}

\newcommand{\bA}{\mathbf{A}}
\newcommand{\bB}{\mathbf{B}}
\newcommand{\bC}{\mathbf{C}}

\newcommand{\bI}{\mathbf{I}}
\newcommand{\bL}{\mathbf{L}}
\newcommand{\bM}{\mathbf{M}}

\newcommand{\bQ}{\mathbf{Q}}

\newcommand{\bJ}{\mathbf{J}}
\newcommand{\eqdef}{:=}

\newcommand{\cD}{{\cal D}}

\newcommand{\cH}{{\cal H}}

\newcommand{\cL}{{\cal L}}

\newcommand{\cX}{{\cal X}}

\usepackage{mdframed} 
\usepackage{thmtools}

\definecolor{shadecolor}{gray}{1.00}
\declaretheoremstyle[
headfont=\normalfont\bfseries,
notefont=\mdseries, notebraces={(}{)},
bodyfont=\normalfont,
postheadspace=0.5em,
spaceabove=1pt,
mdframed={
  skipabove=8pt,
  skipbelow=8pt,
  hidealllines=true,
  backgroundcolor={shadecolor},
  innerleftmargin=4pt,
  innerrightmargin=4pt}
]{shaded}

\declaretheorem[style=shaded,within=section]{definition}
\declaretheorem[style=shaded,sibling=definition]{theorem}
\declaretheorem[style=shaded,sibling=definition]{proposition}
\declaretheorem[style=shaded,sibling=definition]{assumption}
\declaretheorem[style=shaded,sibling=definition]{corollary}

\declaretheorem[style=shaded,sibling=definition]{lemma}

\providecommand{\norm}[1]{\left\| #1\right\|}
\newcommand{\dotprod}[1]{\left< #1\right>}

\usepackage{hyperref}

\graphicspath{{figures/}}

\begin{document}

\twocolumn[
\icmltitle{Stochastic Hamiltonian Gradient Methods for Smooth Games}



\icmlsetsymbol{equal}{*}

\begin{icmlauthorlist}
\icmlauthor{Nicolas Loizou}{to}
\icmlauthor{Hugo Berard}{to,goo}
\icmlauthor{Alexia Jolicoeur-Martineau}{to}
\icmlauthor{Pascal Vincent$^\dagger$}{to,goo}
\icmlauthor{Simon Lacoste-Julien$^\dagger$}{to}
\icmlauthor{Ioannis Mitliagkas$^\dagger$}{to}
\end{icmlauthorlist}

\icmlaffiliation{to}{Mila, Universit\'{e} de Montr\'{e}al $\,\dagger$ Canada CIFAR AI Chair}
\icmlaffiliation{goo}{Facebook AI Research}

\icmlcorrespondingauthor{Nicolas Loizou}{loizouni@mila.quebec}

\icmlkeywords{Machine Learning, ICML}

\vskip 0.3in
]



\printAffiliationsAndNotice{}  % leave blank if no need to mention equal contribution

\begin{abstract}
The success of adversarial formulations in machine learning has brought renewed motivation for smooth games. In this work, we focus on the class of stochastic Hamiltonian methods and provide the first convergence guarantees for certain classes of stochastic smooth games. We propose a novel unbiased estimator for the stochastic Hamiltonian gradient descent (SHGD) and highlight its benefits. Using tools from the optimization literature we show that SHGD converges linearly to the neighbourhood of a stationary point. To guarantee convergence to the exact solution, we analyze SHGD with a decreasing step-size and we also present the first stochastic variance reduced Hamiltonian method. Our results provide the first global non-asymptotic last-iterate convergence guarantees for the class of stochastic unconstrained bilinear games and for the more general class of stochastic games that satisfy a ``sufficiently bilinear" condition, notably including some non-convex non-concave problems. 
We supplement our analysis with experiments on stochastic bilinear and sufficiently bilinear games, where our theory is shown to be tight, and on simple adversarial machine learning formulations.
\end{abstract}

\section{Introduction}
\label{Intro}
We consider the min-max optimization problem
\begin{equation}
\label{MainDeterministicProblem}
\min_{x_1 \in\R^{d_1}} \max_{x_2 \in\R^{d_2}}  g(x_1,x_2) 
\end{equation}
where $g:\R^{d_1} \times \R^{d_2}\rightarrow \R$ is a smooth objective.
Our goal is to find $x^*=(x_1^*, x_2^*)^\top \in \R^{d}$ where $d=d_1+ d_2$ such that
\begin{equation}
g(x_1^*, x_2) \leq g(x_1^*, x_2^*) \leq g(x_1, x_2^*),
\end{equation}
for every $x_1 \in \R^{d_1}$ and $x_2 \in \R^{d_2}$.
We call point, $x^*$, a \emph{ saddle point}, \emph{min-max solution} or {\em Nash equilibrium} of \eqref{MainDeterministicProblem}.
In its general form, this problem is hard.
In this work we focus on the simplest family of problems where some important questions are still open:
the case where all stationary points are global min-max solutions.

Motivated by recent applications in machine learning, we are particularly interested in cases where the objective, $g$, is naturally expressed as a finite sum
\begin{equation}
\label{MainStochasticProblem}
\min_{x_1 \in\R^{d_1}} \max_{x_2 \in\R^{d_2}} g(x_1, x_2) = \frac{1}{n} \sum_{i=1}^n g_i(x_1,x_2) 
\end{equation}
where each component function $g_i:\R^{d_1} \times \R^{d_2}\rightarrow \R$ is assumed to be smooth.
Indeed, in problems like domain generalization \cite{albuquerque2019adversarial}, 
generative adversarial networks \cite{goodfellow2014generative}, 
and some formulations in reinforcement learning \cite{pfau2016connecting},
 empirical risk minimization yields finite sums of the form of \eqref{MainStochasticProblem}.
We  refer to this formulation as a \textbf{{\em stochastic smooth game}}.\footnote{We note that all of our results except the one on variance reduction do not require the finite-sum assumption and can be easily adapted to the stochastic setting (see Appendix~\ref{BeyondFiniteSum}).} We call problem \eqref{MainDeterministicProblem} a {\em deterministic game}.

The deterministic version of the problem has been studied in a number of classic  \cite{korpelevich1976extragradient,nemirovski2004prox} and recent results \cite{mescheder2017numerics,ibrahim2019linear,gidel2018variational,daskalakis2017training,gidel2018negative,mokhtari2020unified,azizian2019tight,azizian2020accelerating} in various settings. 
Importantly, the majority of these results provide 
\textbf{last-iterate convergence} guarantees.
In contrast, for the stochastic setting, guarantees on the classic extragradient method and its variants rely on iterate averaging over compact domains \cite{nemirovski2004prox}.
However, \citet{chavdarova2019reducing} highlighted a possibility of pathological behavior where the iterates \emph{diverge towards}  and then rotate near the boundary of the domain, far from the solution, while their average is shown to converge to the solution (by convexity).\footnote{This is qualitatively very different to stochastic minimization where the iterates converge towards a neighborhood of the solution and averaging is only used to stabilize the method.}
This behavior is also problematic in the context of applying the method on non-convex problems, where averaging do not necessarily yield a solution \cite{daskalakis2017training,abernethy2019last}.
It is only very recently that last-iterate convergence guarantees over a \textbf{non-compact domain} appeared in literature for the stochastic problem   \cite{palaniappan2016stochastic,chavdarova2019reducing,hsieh2019convergence,mishchenko2020revisiting} under the assumption of strong monotonicity. 
Strong monotonicity, a generalization of strong convexity for general operators, seems to be an essential condition for fast convergence in optimization. 
Here, we make \textbf{no strong  monotonicity assumption}. 

The algorithms we consider belong to a recently introduced family of computationally-light second order methods which in each step require the computation of a Jacobian-vector product. Methods that belong to this family are the consensus optimization (CO) method \cite{mescheder2017numerics} and Hamiltonian gradient descent \cite{balduzzi2018mechanics,abernethy2019last}. Even though some convergence results for these methods are known for the deterministic problem, there is no available analysis for the stochastic problem. We close this gap. 
We study {\em stochastic Hamiltonian gradient descent} (SHGD), and propose the first stochastic variance reduced Hamiltonian method, named L-SVRHG. Our contributions are summarized as follows: 
\vspace{-0.1in}
\begin{itemize}[leftmargin=*]
\setlength{\itemsep}{0pt}
\item 
Our results provide the first set of global non-asymptotic last-iterate convergence guarantees for a stochastic game over a non-compact domain, in the absence of strong monotonicity assumptions.
\item The proposed stochastic Hamiltonian methods use \emph{novel unbiased estimators} of the gradient of the Hamiltonian function. This is an essential point for providing convergence guarantees. Existing practical variants of SHGD use biased estimators \cite{mescheder2017numerics}.
\item We provide the first efficient convergence analysis of stochastic Hamiltonian methods. In particular, we focus on solving two classes of stochastic smooth games:
\begin{itemize}
\item \emph{Stochastic Bilinear Games}. 
\item Stochastic games satisfying the ``sufficiently bilinear" condition or simply \emph{Stochastic Sufficiently Bilinear Games}. The deterministic variant of this class of games was firstly introduced by \citet{abernethy2019last} to study the deterministic problem and notably includes some \textbf{non-monotone problems}.
\end{itemize}
\item For the above two classes of games, we provide convergence guarantees for SHGD with a constant step-size (linear convergence to a neighborhood of stationary point), SHGD with a variable step-size (sub-linear convergence to a stationary point) and L-SVRHG. For the latter, we guarantee a linear rate.
\item We show the benefits of the proposed methods by performing numerical experiments on simple stochastic bilinear and sufficiently bilinear problems, as well as toy GAN problems for which the optimal solution is known. Our numerical findings corroborate our theoretical results.
\end{itemize}

\section{Further Related work}
\label{sec:related}
In recent years, several second-order methods have been proposed for solving the min-max optimization problem~\eqref{MainDeterministicProblem}. Some of them require the computation or inversion of a Jacobian which is a highly inefficient operation \cite{wang2019solving,mazumdar2019finding}. In contrast, second-order methods like the ones presented in 
\citet{mescheder2017numerics,balduzzi2018mechanics,abernethy2019last} and in this work are more efficient as they only rely on the computation of a Jacobian-vector product in each step. 

\citet{abernethy2019last} provide the first last-iterate convergence rates for the deterministic Hamiltonian gradient descent (HGD) for several classes of games including games satisfying the sufficiently bilinear condition. The authors briefly touch upon the stochastic setting and by using the convergence results of \citet{karimi2016linear}, explain how a stochastic variant of HGD with decreasing stepsize behaves. Their approach was purely theoretical and they did not provide an efficient way of selecting the unbiased estimators of the gradient of the Hamiltonian. In addition, they assumed bounded gradient of the Hamiltonian function which is restrictive for functions satisfying the Polyak-Lojasiewicz (PL) condition \cite{gower2020sgd}. In this work we provide the first efficient variants and analysis of SHGD. We did that by choosing practical unbiased estimator of the full gradient and by using the recently proposed assumptions of expected smoothness \cite{gower2019sgd} and expected residual \cite{gower2020sgd} in our analysis. The proposed theory of SHGD allow us to obtain as a corollary tight convergence guarantees for the deterministic HGD recovering the result of \citet{abernethy2019last} for the sufficiently bilinear games.

In another line of work, \citet{carmon2019variance} analyze variance reduction methods for constrained finite-sum problems and \citet{ryu2019ode} provide an ODE-based analysis and guarantees in the monotone but potentially non-smooth case.
\citet{chavdarova2019reducing} show that both alternate stochastic descent-ascent and stochastic extragradient diverge on an unconstrained stochastic bilinear problem. In the same paper, \citet{chavdarova2019reducing} propose the stochastic variance reduced extragradient (SVRE) algorithm with restart, which empirically achieves last-iterate convergence on this problem. However, it came with no theoretical guarantees. In Section~\ref{sec:experiments}, we observe in our experiments that SVRE is slower than the proposed L-SVRHG for both the stochastic bilinear and sufficiency bilinear games that we tested.

In concurrent work, \citet{yang2020global} provide global convergence guarantees for stochastic alternate gradient descent-ascent (and its variance reduction variant) for a subclass of nonconvex-nonconcave objectives satisfying a so-called two-sided Polyak-Lojasiewicz inequality, but this does not include the stochastic bilinear problem that we cover.

\section{Technical Preliminaries}
\label{sec:preliminaries}
In this section, we present the necessary background and the basic notation used in the paper. We also describe the update rule of the deterministic Hamiltonian method.
\subsection{Optimization Background: Basic Definitions}
\label{optBackground}
We start by presenting some definitions that we will later use in the analysis of the proposed methods.
\begin{definition}
\label{QSCdefinition}
Function $f : \R^d \rightarrow \R$ is $\mu$--quasi-strongly convex if there exists a constant
$\mu > 0$ such that $\forall x \in \R^d$: 
$f^* \geq f(x)+ \dotprod{\nabla f(x) , x^*-x} + \tfrac{\mu}{2} \norm{x^*-x}^2,$
where $f^*$ is the minimum value of $f$ and $x^*$ is the projection of $x$ onto the solution set $\mathcal{X}^*$ minimizing~$f$.
\end{definition}
\begin{definition}
\label{Polyak}
We say that a function satisfies the Polyak-Lojasiewicz (PL) condition if there exists $\mu >0$ such that 
\begin{equation}
\label{PLcondition}
\frac{1}{2}\|\nabla f(x)\|^2 \geq \mu \left[f(x)-f^*\right] \quad \forall x \in \R^d \, ,
\end{equation}
where $f^*$ is the minimum value of $f$.
\end{definition}
An analysis of several stochastic optimization methods under the assumption of PL condition \citep{polyak1987introduction} was recently proposed in \citet{karimi2016linear}. A function can satisfy the PL condition and not be strongly convex, or even convex.
However, if the function is $\mu-$quasi strongly convex then it satisfies the PL condition with the same $\mu$ \cite{karimi2016linear}. 
\begin{definition}
\label{Lsmooth}
Function $f:\R^d\rightarrow \R$ is $L$-smooth if there exists $L >0$ such that:\\ ${\|\nabla f(x) -\nabla f(y)\| \leq L \|x-y\|} \quad \forall x, y \in \R^d$.
\end{definition}
If $f=\frac{1}{n} \sum_{i=1}^n f_i(x)$, then a more refined analysis of stochastic gradient methods has been proposed under new notions of smoothness. In particular, the notions of \emph{expected smoothness (ES)} and \emph{expected residual (ER)} have been introduced and used in the analysis of SGD in \citet{gower2019sgd} and \citet{gower2020sgd} respectively.  ES and ER are generic and remarkably weak assumptions. In Section~\ref{sec:analysis} and Appendix~\ref{ESandER}, we provide more details on their generality. We state their definitions below. 
\begin{definition}[Expected smoothness, \citep{gower2019sgd}]
\label{ass:Expsmooth} We say that the function $f=\frac{1}{n} \sum_{i=1}^n f_i(x)$ satisfies the \emph{expected smoothness} condition if there exists $\cL>0$ such that for all $x\in\R^d$,
\begin{equation}
\label{eq:expsmooth}
\EE{i}{\norm{\nabla f_i(x)-\nabla f_i(x^*)}^2} \leq 2\cL (f(x)-f(x^*)),
\end{equation}
\end{definition}
\begin{definition}[Expected residual,  \citep{gower2020sgd}]\label{ass:expresidual} 
We say that the function $f=\frac{1}{n} \sum_{i=1}^n f_i(x)$ satisfies the \emph{expected residual} condition if there exists $\rho  >0$ such that for all $x\in\R^d$,
\begin{multline}
\label{eq:expresidual}
\EE{i}{\norm{\nabla f_i(x)-\nabla f_i(x^*) -  ( \nabla f(x)-\nabla f(x^*))}^2} \\ 
\leq 2\rho\left(f(x)-f(x^*) \right).
\end{multline}
\end{definition}

\subsection{Smooth Min-Max Optimization}
We use standard notation used previously in \citet{mescheder2017numerics,balduzzi2018mechanics,abernethy2019last,letcher2019differentiable}.

Let $x=(x_1, x_2)^\top \in \R^{d}$ be the column vector obtained by stacking $x_1 $ and $ x_2 $ one on top of the other. With $\xi(x):=\left(\nabla_{x_1} g, -\nabla_{x_2} g \right)^\top$, we denote the signed vector of partial derivatives evaluated at point $x$.
Thus, $\xi(x):\R^d \rightarrow \R^d$ is a vector function.
We use
$$\bJ=\nabla \xi=\begin{pmatrix} 
\nabla^2_{x_1,x_1} g & \nabla^2_{x_1,x_2} g \\ 
-\nabla^2_{x_2,x_1} g & -\nabla^2_{x_2,x_2} g  \\ 
\end{pmatrix} \in \R^{d \times d}
$$
to denote the Jacobian of the vector function $\xi$.
Note that using the above notation, the simultaneous gradient descent/ascent (SGDA) update can be written simply as: $x^{k+1}=x^k - \eta_k \xi(x_k)$.
\begin{definition}
The objective function $g$ of problem \eqref{MainDeterministicProblem} is $L_g$-smooth if there exist $L_g>0$ such that: \\ $\|\xi(x)-\xi(y)\| \leq L_g \|x-y\| \quad \forall x, y \in \R^d$. 

We also say that $g$ is $L$-smooth in $x_1$ (in $x_2$) if $\|\nabla_{x_1} g(x_1,x_2)-\nabla_{x_1} g(x_1',x_2)\| \leq L \|x_1-x_1'\|$ (if $\|\nabla_{x_2} g(x_1,x_2)-\nabla_{x_2} g(x_1,x_2')\| \leq L \|x_2-x_2'\|$) $\quad \text{for all } x_1, x_1' \in \R^{d_1}$ ($ \text{for all } x_2, x_2' \in \R^{d_2}$).
\end{definition}
\begin{definition}
A stationary point of function $f:\R^{d} \rightarrow \R$ is a point $x^* \in \R^d$ such that $\nabla f(x^*)=0$. Using the above notation, in min-max problem \eqref{MainDeterministicProblem}, point $x^* \in \R^d$ is a stationary point when $\xi(x^*)=0$.
\end{definition}

As mentioned in the introduction, in this work we focus on smooth games satisfying the following assumption.
\begin{assumption}
\label{criticalminmax}
The objective function $g$ of problem~\eqref{MainStochasticProblem} has at least one stationary point and all of its stationary points are global min-max solutions.
\end{assumption}

With Assumption~\ref{criticalminmax}, we can guarantee convergence to a min-max solution of problem \eqref{MainStochasticProblem} by proving convergence to a stationary point. This assumption is true for several classes of games including strongly convex-strongly concave and convex-concave games. However, it can also be true for some classes of non-convex non-concave games~\cite{abernethy2019last}. In Section~\ref{sec:hamiltonian}, we describe in more details the two classes of games that we study. Both satisfy this assumption.

\subsection{Deterministic Hamiltonian Gradient Descent}
Hamiltonian gradient descent (HGD) has been proposed as an efficient method for solving min-max problems in \citet{balduzzi2018mechanics}.  To the best of our knowledge, the first convergence analysis of the method is presented in \citet{abernethy2019last} where the authors prove non-asymptotic linear last-iterate convergence rates for several classes of games.

In particular, HGD converges to saddle points of problem \eqref{MainDeterministicProblem} by performing gradient descent on a particular objective function $\cH$, which is called the Hamiltonian function \citep{balduzzi2018mechanics}, and has the following form:
\begin{equation}
\label{HamiltonianProblem}
 \min_x \quad \cH(x)= \frac{1}{2} \|\xi(x)\|^2.
\end{equation}

That is, HGD is a gradient descent method that minimizes the square norm of the gradient $\xi(x)$.
Note that under Assumption~\ref{criticalminmax}, solving problem \eqref{HamiltonianProblem} is equivalent to solving problem \eqref{MainDeterministicProblem}.
The equivalence comes from the fact that 
$\cH$ only achieves its minimum at stationary points. 
The update rule of HGD can be expressed using a Jacobian-vector product \citep{balduzzi2018mechanics,abernethy2019last}:
\begin{equation}
\label{DetHamiltonian}
x^{k+1}=x^k - \eta_k \nabla \cH(x)= x^k - \eta_k \left[ \bJ^\top \xi \right],
\end{equation}
making HGD a second-order method.
However, as discussed in~\citet{balduzzi2018mechanics}, the Jacobian-vector product can be efficiently evaluated in tasks like training neural networks and the computation time of the gradient and the Jacobian-vector product is comparable~\cite{pearlmutter1994fast}.

\section{Stochastic Smooth Games and Stochastic Hamiltonian Function}
\label{sec:hamiltonian}
In this section, we provide the two classes of stochastic games that we study. We define the stochastic counterpart to the Hamiltonian function as a step towards solving problem~\eqref{MainStochasticProblem} and present its main properties.

Let us start by presenting the basic notation for the stochastic setting.
Let $\xi(x)=\frac{1}{n}\sum_{i=1}^n \xi_i(x),$
where $\xi_i(x):=\left(\nabla_{x_1} g_i, -\nabla_{x_2} g_i \right)^\top$, for all $i \in [n]$ and 
let 
$$\bJ=\frac{1}{n}\sum_{i=1}^n \bJ_i,
\quad\textrm{where\ }
\bJ_i=\begin{pmatrix} 
\nabla^2_{x_1,x_1} g_i & \nabla^2_{x_1,x_2} g_i \\ 
-\nabla^2_{x_2,x_1} g_i & -\nabla^2_{x_2,x_2}  g_i  \\ 
\end{pmatrix}.$$
Using the above notation, the stochastic variant of SGDA can be written as $x^{k+1}=x^k - \eta_k \xi_i(x_k)$
where $\Exp_i[ \xi_i(x_k)]= \xi(x_k)$.\footnote{Here the expectation is over the uniform distribution. That is, $\Exp_i[ \xi_i(x)]=\frac{1}{n}\sum_{i=1}^n \xi_i(x)$.}

In this work, we focus on stochastic smooth games of the form~\eqref{MainStochasticProblem} that satisfy the following assumption.
\begin{assumption}
\label{AssumptionOnGi}
Functions $g_i:\R^{d_1} \times \R^{d_2}\rightarrow \R$ of problem \eqref{MainStochasticProblem} are twice differentiable, $L_i$-smooth with $S_i$-Lipschitz Jacobian. That is, for each $i \in[n]$ there are constants $L_i >0$ and $S_i >0$ such that $\|\xi_i(x)-\xi_i(y)\| \leq L_i \|x-y\|$ and $\|\bJ_i(x)-\bJ_i(y)\| \leq S_i \|x-y\|$ for all $x ,y \in R^d$.
\end{assumption}

\subsection{Classes of Stochastic Games}
\label{theclasses}
Here we formalize the two families of stochastic smooth games under study: (i) stochastic bilinear, and (ii) stochastic sufficiently bilinear. Both families satisfy Assumption~\ref{criticalminmax}. Interestingly, the latter family includes some non-convex non-concave games, i.e. non-monotone problems.

\paragraph{Stochastic Bilinear Games.}
A stochastic bilinear game is the stochastic smooth game \eqref{MainStochasticProblem} in which function $g$ has the following structure:
\begin{equation}
\label{bilinearGame1}
 g(x_1,x_2)=\frac{1}{n} \sum_{i=1}^n \left( x_1^\top b_i+x_1^\top \bA_i x_2 +c_i^\top x_2 \right) \, .
\end{equation}
While this game appears simple, standard methods diverge on it~\citep{chavdarova2019reducing} and L-SVRHG gives the first stochastic method with last-iterate convergence guarantees. 
\paragraph{Stochastic sufficiently bilinear games.}
A game of the form~\eqref{MainStochasticProblem} is called \emph{stochastic sufficiently bilinear} if it satisfies the following definition.
\begin{definition} 
\label{SuffBilinear}
Let Assumption~\ref{AssumptionOnGi} be satisfied and let the objective function $g$ of problem~\eqref{MainStochasticProblem} be $L$-smooth in $x_1$ and $L$-smooth in $x_2$.  Assume that a constant $C>0$ exists, such that $\Exp_i\|\xi_i(x)\|<C$. Assume the cross derivative $\nabla^2_{x_1,x_2} g$ be full rank with $0 < \delta \leq \sigma_i \left(\nabla^2_{x_1,x_2} g\right)  \leq \Delta$ for all $x \in \R^d$ and for all singular values $\sigma_i $. Let $\rho^2 = \min_{x_1,x_2} \lambda_{\min} \left[ \nabla^2_{x_1,x_1} g(x_1,x_2)\right]^2$ and $\beta^2 = \min_{x_1,x_2} \lambda_{\min} \left[ \nabla^2_{x_2,x_2} g(x_1,x_2)\right]^2$. Finally let the following condition to be true:
\begin{equation}
\label{SufficientBilinear}
(\delta^2 +\rho^2)(\delta^2 +\beta^2) -4L^2 \Delta^2 >0.
\end{equation}
\end{definition} 
Note that the definition of the stochastic sufficiently bilinear game has no restriction on the convexity of functions $g_i(x)$ and $g(x)$. The most important condition that needs to be satisfied is the expression in equation~\eqref{SufficientBilinear}. Following the terminology of \citet{abernethy2019last}, we call the condition~\eqref{SufficientBilinear}: ``\emph{sufficiently bilinear}" condition. Later in our numerical evaluation, we present stochastic non convex-non concave min-max problems that satisfy condition~\eqref{SufficientBilinear}. 

We highlight that the deterministic counterpart of the above game was first proposed in~\citet{abernethy2019last}. The deterministic variant of \citet{abernethy2019last} can be obtained as special case of the above class of games when $n=1$ in problem~\eqref{MainStochasticProblem}.

\subsection{Stochastic Hamiltonian Function}
\label{stoHamilFunction}
Having presented the two main classes of stochastic smooth games, in this section we focus on the structure of the stochastic Hamiltonian function and highlight some of its properties.

\paragraph{Finite-Sum Structure Hamiltonian Function.}
Having the objective function $g$ of problem~\eqref{MainStochasticProblem} to be stochastic and in particular to be a finite-sum function,  leads to the following expression for the Hamiltonian function:
\begin{eqnarray}
\label{StochHamiltonianFunction}
\cH(x)=\frac{1}{n^2} \sum_{i,j=1}^n \underbrace{\frac{1}{2} \langle  \xi_i(x),  \xi_j(x)\rangle}_{\cH_{i,j}(x)} \, .
\end{eqnarray}
That is, the Hamiltonian function $\cH(x)$ can be expressed as a finite-sum with $n^2$ components.

\paragraph{Properties of the Hamiltonian Function.}
As we will see in the following sections, the finite-sum structure of the stochastic Hamiltonian function \eqref{StochHamiltonianFunction} allows us to use popular stochastic optimization problems for solving problem~\eqref{HamiltonianProblem}. However in order to be able to provide convergence guarantees of the proposed stochastic Hamiltonian methods, we need to show that the stochastic Hamiltonian function \eqref{StochHamiltonianFunction} satisfies specific properties for the two classes of games we study. This is what we do in the following two propositions.
\begin{proposition}
\label{BilinearGameProposition}
For stochastic bilinear games of the form~\eqref{bilinearGame1}, the stochastic Hamiltonian function~\eqref{StochHamiltonianFunction} is a smooth quadratic $\mu_{\cH}$--quasi-strongly convex function with constants $L_{\cH}=\sigma_{\max}^2 (\bA)$ and $\mu_{\cH} = \sigma_{\min}^2(\bA)$ where $\bA=\frac{1}{n} \sum_{i=1}^n \bA_i$ and $\sigma_{\max}$ and $\sigma_{\min}$ are the maximum and minimum non-zero singular values of $\bA$.
\end{proposition}
\begin{proposition}
\label{SufficientlyBilinearGameProposition}
For stochastic sufficiently bilinear games, the stochastic Hamiltonian function~\eqref{StochHamiltonianFunction} is a $L_{\cH}= \bar{S} C+\bar{L}^2$ smooth function and satisfies the PL condition~\eqref{PLcondition} with $\mu_{\cH}=\frac{(\delta^2 +\rho^2)(\delta^2 +\beta^2) -4L^2 \Delta^2 }{2\delta^2+\rho^2+\beta^2}$. Here $\bar{S}=\Exp_i[S_i]$ and $\bar{L}=\Exp_i[L_i]$.
\end{proposition}

\section{Stochastic Hamiltonian Gradient Methods}
\label{sec:methods}

In this section we present the proposed stochastic Hamiltonian methods for solving the stochastic min-max problem~\eqref{MainStochasticProblem}. Our methods could be seen as extensions of popular stochastic optimization methods into the Hamiltonian setting. In particular, the two algorithms that we build upon are the popular stochastic gradient descent (SGD) and the recently introduced loopless stochastic variance reduced gradient (L-SVRG). For completeness, we present their form for solving finite-sum optimization problems in Appendix~\ref{AppendixTechnical}.

\subsection{Unbiased Estimator}
\label{unbiased-estimator}
One of the most important elements of stochastic gradient-based optimization algorithms for solving finite-sum problems of the form~\eqref{StochHamiltonianFunction} is the selection of unbiased estimators of the full gradient $\nabla \cH(x)$ in each step.  In our proposed optimization algorithms for solving~\eqref{StochHamiltonianFunction}, at each step we use the gradient of only one component function $\cH_{i,j}(x)$:
\begin{eqnarray}
\nabla \cH_{i,j}(x)=\frac{1}{2} \left[ \bJ_i^\top \xi_j +   \bJ_j^\top \xi_i  \right].
\end{eqnarray}
It can easily be shown that this selection is an unbiased estimator of $\nabla \cH(x)$. That is, $\Exp_{i,j}\left[\nabla \cH_{i,j}(x)\right]= \nabla \cH(x).$

\subsection{Stochastic Hamiltonian Gradient Descent (SHGD)}
Stochastic gradient descent (SGD) \cite{robbins1951stochastic, NemYudin1978, NemYudin1983book, Nemirovski-Juditsky-Lan-Shapiro-2009, HardtRechtSinger-stability_of_SGD, gower2019sgd, gower2020sgd, loizou2020stochastic} is the workhorse for training supervised machine learning problems. In Algorithm~\ref{SHGD_Algorithm}, we apply SGD to~\eqref{StochHamiltonianFunction}, yielding stochastic Hamiltonian gradient descent (SHGD) for solving problem~\eqref{MainStochasticProblem}. Note that at each step, $i \sim {\cal D}$ and  $j \sim {\cal D}$ are sampled from a given well-defined distribution ${\cal D}$ and then are used to evaluate $\nabla \cH_{i,j}(x^k)$ (unbiased estimator of the full gradient). In our analysis, we provide rates for two selections of step-sizes for SHGD. These are the constant step-size $\gamma^k=\gamma$ and the decreasing step-size (switching rule which
describe when one should switch from a constant to a decreasing stepsize regime).
\begin{algorithm}[tb]
   \caption{Stochastic Hamiltonian Gradient Descent (SHGD)}
   \label{SHGD_Algorithm}
\begin{algorithmic}
   \STATE {\bfseries Input:} Starting stepsize $\gamma^0>0$. Choose initial points $x^0 \in \R^d$. Distribution $\cD$ of samples.
   %\REPEAT
   \FOR{$k=0,1,2,\cdots, K$}
   \STATE Generate fresh samples $i \sim {\cal D}$ and  $j \sim {\cal D}$ and evaluate $\nabla \cH_{i,j}(x^k)$.
   \STATE Set step-size $\gamma^k$ following one of the selected choices (constant, decreasing)
   \STATE Set $x^{k+1}=x^k -\gamma^k \nabla \cH_{i,j}(x^k)$
   \ENDFOR
\end{algorithmic}
\end{algorithm}

\subsection{Loopless Stochastic Variance Reduced Hamiltonian Gradient (L-SVRHG)}
One of the main disadvantage of Algorithm~\ref{SHGD_Algorithm} with constant step-size selection is that it guarantees  linear convergence only to a neighborhood of the min-max solution $x^*$. As we will present in Section~\ref{sec:analysis}, the decreasing step-size selection allow us to obtain exact convergence to the min-max but at the expense of slower rate (sublinear). 

One of the most remarkable algorithmic breakthroughs in recent years was the development of variance-reduced stochastic gradient algorithms for solving finite-sum optimization problems. These algorithms, by reducing the variance of the stochastic gradients, are able to guarantee convergence to the exact solution of the optimization problem with faster convergence than classical SGD. For example, for smooth strongly convex functions, variance reduced methods can guarantee linear convergence to the optimum. This is a vast improvement on the sub-linear convergence of SGD with decreasing step-size. In the past several years, many efficient variance-reduced methods have been proposed. Some popular examples of variance reduced algorithms are SAG \cite{schmidt2017minimizing}, SAGA \cite{defazio2014saga}, SVRG \cite{johnson2013accelerating} and SARAH \cite{nguyen2017sarah}. For more examples of variance reduced methods in different settings, see \citet{defazio2016simple, mS2GD, GowerRichBach2018, sebbouh2019towards}.

In our second method Algorithm~\ref{LSVRHG_Algorithm}, we propose a variance reduced Hamiltonian method for solving~\eqref{MainStochasticProblem}. Our method is inspired by the recently introduced and well behaved variance reduced algorithm, Loopless-SVRG (L-SVRG) first proposed in \citet{hofmann2015variance, kovalev2019don} and further analyzed under different settings in \citet{qian2019svrg, gorbunov2020unified, khaled2020unified}. We name our method loopless stochastic variance reduced Hamiltonian gradient (L-SVRHG). The method works by selecting at each step the unbiased estimator $g^k= \nabla \cH_{i,j}(x^k)-\nabla \cH_{i,j}(w^k)+\nabla \cH(w^k)$ of the full gradient. As we will prove in the next section, this method guarantees linear convergence to the min-max solution of the stochastic bilinear game~\eqref{bilinearGame1}. 

\begin{algorithm}[tb]
   \caption{Loopless Stochastic Variance Reduced Hamiltonian Gradient (L-SVRHG)}
   \label{LSVRHG_Algorithm}
\begin{algorithmic}
   \STATE {\bfseries Input:} Starting stepsize $\gamma>0$. Choose initial points $x^0=w^0 \in \R^d$. Distribution $\cD$ of samples. Probability $p \in (0,1]$
   \FOR{$k=0,1,2,\cdots, K-1$}
   \STATE Generate fresh samples $i \sim {\cal D}$ and  $j \sim {\cal D}$ and evaluate $\nabla \cH_{i,j}(x^k)$.
   \STATE Evaluate $g^k= \nabla \cH_{i,j}(x^k)-\nabla \cH_{i,j}(w^k)+\nabla \cH(w^k)$.
   \STATE Set $x^{k+1}=x^k -\gamma g^k$
   \STATE Set $$ w^{k+1} = \begin{cases} x^k \quad \text{with probability} \quad p\\  w^k \quad \text{with probability} \quad 1-p \end{cases}$$
   \ENDFOR
   \STATE {\bf Output:} \\
   Option I: The last iterate $x=x^k$.\\
   Option II: $x$ is chosen uniformly at random from $\{x^i\}^K_{i=0}$.
\end{algorithmic}
\end{algorithm}

To get a linearly convergent algorithm in the more general setup of sufficiently bilinear games~\ref{SuffBilinear}, we had to propose  a restarted variant of Alg.~\ref{LSVRHG_Algorithm}, presented in Alg.~\ref{PL-LSVRHG_Algorithm}, which calls at each step Alg.~\ref{LSVRHG_Algorithm} with the second option of output, that is L-SVRHG$_{II}$. Using the property from Proposition~\ref{SufficientlyBilinearGameProposition} that the Hamiltonian function~\eqref{StochHamiltonianFunction} satisfy the PL condition~\ref{Polyak}, we show that Alg.~\ref{PL-LSVRHG_Algorithm} converges linearly to the solution of the sufficiently bilinear game (Theorem~\ref{TheoremLSVRHGforPL}).

\begin{algorithm}[tb]
   \caption{L-SVRHG (with Restart)}
   \label{PL-LSVRHG_Algorithm}
\begin{algorithmic}
     \STATE {\bfseries Input:} Starting stepsize $\gamma>0$. Choose initial points $x^0=w^0 \in \R^d$. Distribution $\cD$ of samples. Probability $p \in (0,1]$, $T$
   %\REPEAT
   \FOR{$t=0,1,2,\cdots, T$}
   \STATE Set $x^{t+1} $ = L-SVRHG$_{II}(x^{t}, K, \gamma, p \in (0,1]$)
   \ENDFOR
   \STATE {\bf Output:} The last iterate $x^T$.
\end{algorithmic}
\end{algorithm}

\section{Convergence Analysis}
\label{sec:analysis}
\label{ConvergenceAnalysis}
We provide theorems giving the performance of the previously described stochastic Hamiltonian methods for solving the two classes of stochastic smooth games: stochastic bilinear and stochastic sufficiently bilinear. In particular, we present three main theorems for each one of these classes describing the convergence rates for (i) SHGD with constant step-size, (ii) SHGD with decreasing step-size and (iii) L-SVRHG and its restart variant (Algorithm~\ref{PL-LSVRHG_Algorithm}). 

The proposed results depend on the two main parameters $\mu_{\cH}$, $L_{\cH}$ evaluated in Propositions~\ref{BilinearGameProposition} and~\ref{SufficientlyBilinearGameProposition}. In addition, the theorems related to the bilinear games (the Hamiltonian function is quasi-strongly convex) use the expected smoothness constant $\cL$~\eqref{eq:expsmooth}, while the theorems related to the sufficiently bilinear games (the Hamiltonian function satisfied the PL condition) use the expected residual constant $\rho$~\eqref{eq:expresidual}. 
We note that the expected smoothness and expected residual constants can take several values according to the well-defined distributions $\cD$ selected in our algorithms and the proposed theory will still hold \cite{gower2019sgd, gower2020sgd}. 

As a concrete example, in the case of $\tau$-minibatch sampling,\footnote{In each step we draw uniformly at random $\tau$ components of the $n^2$ possible choices of the stochastic Hamiltonian function~\eqref{StochHamiltonianFunction}. For more details on the $\tau$-minibatch sampling see Appendix~\ref{ESandER}.} the expected smoothness and expected residual parameters take the following values:
\begin{gather} 
\label{nic1}
 \cL(\tau) = \tfrac{n^2(\tau-1)}{\tau(n^2-1)}L_{\cH} + \tfrac{n^2-\tau}{\tau(n^2-1)}L_{\max}\\ 
 \label{nic2}
\rho(\tau)=L_{\max} \tfrac{n^2-\tau}{(n^2-1)\tau}
\end{gather}
where $L_{\max}=\max_{\{1,\dots,n^2\}} \{L_{\cH_{i,j}}\}$ is the maximum smoothness constant of the functions $\cH_{i,j}$.
By using the expressions \eqref{nic1} and \eqref{nic2}, it is easy to see that for single element sampling where $\tau=1$ (the one we use in our experiments)  $\cL=\rho=L_{\max}$. On the other limit case where a full-batch is used ($\tau=n^2$), that is we run the deterministic Hamiltonian gradient descent, these values become $\cL=L_{\cH}$ and $\rho=0$ and as we explain below, the proposed theorems include the convergence of the deterministic method as special case. 

\subsection{Stochastic Bilinear Games}
We start by presenting the convergence of SHGD with constant step-size and explain how we can also obtain an analysis of the HGD \eqref{DetHamiltonian} as special case. Then we move to the convergence of SHGD with decreasing step-size and the L-SVRHG where we are able to guarantee convergence to a min-max solution $x^*$. In the results related to SHGD we use $\sigma^2  \eqdef \Exp_{i,j}[\norm{\nabla \cH_{i,j}(x^*)}^2]$ to denote the finite gradient noise at the solution.
\begin{theorem}[Constant stepsize]
\label{theo:strcnvlin}
Let us have the stochastic bilinear game \eqref{bilinearGame1}. Then iterates of SHGD with constant step-size $\gamma^k=\gamma \in (0,  \frac{1}{2\cL}]$ satisfy: 
\begin{equation}\label{eq:convsgd}
\mathbb{E} \| x^k - x^* \|^2 \leq \left( 1 - \gamma \mu_{\cH} \right)^k \| x^0 - x^* \|^2 + \frac{2 \gamma \sigma^2}{\mu}.
\end{equation}
\end{theorem}
That is, Theorem~\ref{theo:strcnvlin} shows linear convergence to a neighborhood of the min-max solution.
Using Theorem~\ref{theo:strcnvlin} and following the approach of~\citet{gower2019sgd}, we can obtain the following corollary on the convergence of deterministic Hamiltonian gradient descent (HGD)~\eqref{DetHamiltonian}. Note that for the deterministic case $\sigma=0$ and $\cL=L$ \eqref{nic1}.
\begin{corollary}
\label{aoskjnda}
Let us have a deterministic bilinear game. Then the iterates of HGD with step-size $\gamma=\frac{1}{2L}$ satisfy: 
\begin{equation}
\| x^k - x^* \|^2 \leq \left( 1 - \gamma \mu_{\cH} \right)^k \| x^0 - x^* \|^2
\end{equation}
\end{corollary}
To the best of our knowledge, Corollary~\ref{aoskjnda} provides the first linear convergence guarantees for HGD in terms of  $\| x^k - x^* \|^2$ (\citet{abernethy2019last} gave guarantees only on $\cH(x^k)$). 
Let us now select a decreasing step-size rule (switching strategy) that guarantees a sublinear convergence to the exact min-max solution for the SHGD.
\begin{theorem}[Decreasing stepsizes/switching strategy]
\label{theo:decreasingstep}
Let us have the stochastic bilinear game \eqref{bilinearGame1}. Let  $\mathcal{K} \eqdef \left.\cL\right/\mu_{\cH}$. Let 
\begin{equation}\label{eq:gammakdef}
\gamma^k= 
\begin{cases}
\displaystyle \tfrac{1}{2\cL} & \mbox{for}\quad k \leq 4\lceil\mathcal{K} \rceil \\[0.3cm]
\displaystyle \tfrac{2k+1}{(k+1)^2 \mu_{\cH}} &  \mbox{for}\quad k > 4\lceil\mathcal{K} \rceil.
\end{cases}
\end{equation}
If $k \geq 4 \lceil\mathcal{K} \rceil$, then SHGD given in Algorithm~\ref{SHGD_Algorithm} satisfy:
\begin{equation}\label{eq:rateofdecreasing}
\mathbb{E}\| x^{k} - x^*\|^2 \le   \tfrac{\sigma^2 }{\mu_{\cH}^2 }\tfrac{8 }{k} + \tfrac{16 \lceil\mathcal{K} \rceil^2}{e^2 k^2 }  \|x^0 - x^*\|^2.\end{equation}
\end{theorem}
Lastly, in the following theorem, we show under what selection of step-size L-SVRHG convergences linearly to a min-max solution.
\begin{theorem}[L-SVRHG]
\label{TheoremLSVRHGBilinear}
Let us have the stochastic bilinear game \eqref{bilinearGame1}. 
Let step-size $\gamma= 1/6L_{\cH}$ and $p \in (0,1]$. Then L-SVRHG with Option I for output as given in Algorithm~\ref{LSVRHG_Algorithm} convergences linearly to the min-max solution $x^*$ and satisfies:
$$\Exp[\Phi^k]\leq \max \left\{ 1-\frac{\mu}{6L_{\cH}} , 1-\frac{p}{2} \right\}^k \Phi^0$$
where $\Phi^k :=\|x^k-x^*\|^2+\frac{4\gamma^2}{p n^2}\sum_{i,j=1}^{n}\|\nabla \cH_{i,j}(w^k)-\nabla \cH_{i,j}(x^*)\|^2$.
\end{theorem}
\subsection{Stochastic Sufficiently-Bilinear Games}
\label{sufficiently-bilinear-games}
As in the previous section, we start by presenting the convergence of SHGD with constant step-size and explain how we can  obtain an analysis of the HGD \eqref{DetHamiltonian} as special case. Then we move to the convergence of SHGD with decreasing step-size and the L-SVRHG (with restart) where we are able to guarantee linear convergence to a min-max solution $x^*$.
In contrast to the results on bilinear games, the convergence guarantees of the following theorems are given in terms of the Hamiltonian function $\Exp[\cH(x^{k})]$. In all theorems we call ``sufficiently-bilinear game" the game described in Definition~\ref{SuffBilinear}.  With $\sigma^2  \eqdef \Exp_{i,j}[\norm{\nabla \cH_{i,j}(x^*)}^2]$, we denote the finite gradient noise at the solution.
\begin{theorem}
\label{SGDforPolyak}
Let us have a stochastic sufficiently-bilinear game. Then the iterates of SHGD with constant steps-size $\gamma^k=\gamma \leq \frac{\mu}{L (\mu +2\rho)}$ satisfy: 
\begin{equation}\label{functionTheorem}
\Exp[\cH(x^{k})] \leq \left(1- \gamma \mu_{\cH} \right)^k [\cH(x^{0})] + \frac{L_{\cH} \gamma \sigma^2} {\mu_{\cH}}.
\end{equation}
\end{theorem}
Using the above Theorem and by following the approach of \citet{gower2020sgd}, we can obtain the following corollary on the convergence of deterministic Hamiltonian gradient descent (HGD) \eqref{DetHamiltonian}. It shows linear convergence of HGD to the min-max solution. Note that for the deterministic case $\sigma=0$ and $\rho=0$ \eqref{nic2}.
\begin{corollary}
\label{ncaoskla}
Let us have a deterministic sufficiently-bilinear game. Then the iterates of HGD with step-size $\gamma=\frac{1}{L_{\cH}}$ satisfy: 
\begin{equation}
\cH(x^{k}) \leq \left( 1 - \gamma \mu_{\cH} \right)^k \cH(x^{0})
\end{equation}
The result of Corollary~\ref{ncaoskla} is equivalent to the convergence of HGD as proposed in \citet{abernethy2019last}. 
\end{corollary}
Let us now show that with decreasing step-size (switching strategy), SHGD can converge (with sub-linear rate) to the min-max solution.
\begin{theorem}[Decreasing stepsizes/switching strategy]
\label{theo:decreasingstepPL}
Let us have a stochastic sufficiently-bilinear game. Let $k^* \eqdef 2\tfrac{L}{\mu} \left(1+2\tfrac{\rho}{\mu}\right)$ and 
\begin{equation}
\gamma^k= 
\begin{cases}
\displaystyle \tfrac{\mu_{\cH}}{L_{\cH} (\mu_{\cH} +2\rho)} & \mbox{for}\quad k \leq \lceil k^*\rceil\\[0.3cm]
\displaystyle \tfrac{2k+1}{(k+1)^2 \mu_{\cH}} &  \mbox{for} \quad k >  \lceil k^* \rceil.
\end{cases}
\end{equation}
If $k \geq  \lceil k^*  \rceil$, then SHGD given in Algorithm~\ref{SHGD_Algorithm} satisfy:
\begin{equation*}
\Exp[\cH(x^{k})] \le   \tfrac{4 L_{\cH} \sigma^2 }{\mu_{\cH}^2 }\tfrac{1}{k} + \tfrac{( k^*)^2}{k^2 e^2}  [\cH(x^{0})] .
\end{equation*}
\end{theorem}
In the next Theorem we show how the updates of L-SVRHG with Restart (Algorithm~\ref{PL-LSVRHG_Algorithm}) converges linearly to the min-max solution. We highlight that each step $t$ of Alg.~\ref{PL-LSVRHG_Algorithm} requires $K=\frac{4}{\mu_{\cH} \gamma}$ updates of the L-SVRHG.
\begin{theorem}[L-SVRHG with Restart]
\label{TheoremLSVRHGforPL}
Let us have a stochastic sufficiently-bilinear game. Let  $p \in (0,1]$ and $
\gamma \leq \min \left\{  \frac{1}{4L_{\cH}}, \frac{p^{2/3}}{36^{1/3}(L_{\cH}  \rho)^{1/3}},\frac{\sqrt{p}}{\sqrt{6\rho}} \right\}
$
and let $K=\frac{4}{\mu_{\cH} \gamma}$. Then the iterates of L-SVRHG (with Restart) given in Algorithm~\ref{PL-LSVRHG_Algorithm} satisfies 
$$\Exp[\cH(x^{t})] \le \left(1/2\right)^t  [\cH(x^{0})].$$
\end{theorem}

\section{Numerical Evaluation}
\label{sec:experiments}

In this section, we compare the algorithms proposed in this paper to existing methods in the literature. Our goal is to illustrate the good convergence properties of the proposed algorithms as well as to explore how these algorithms behave in settings not covered by the theory.
We propose to compare the following algorithms: \textbf{SHGD} with constant step-size and decreasing step-size, a biased version of SHGD \citep{mescheder2017numerics}, \textbf{L-SVRHG} with and without restart, consensus optimization (\textbf{CO})\footnote{\textbf{CO} is a mix between SGDA and  SHGD, with the following update rule $x^{k+1}=x^k - \eta_k (\xi_i(x^k) + \lambda \nabla \cH_{i,j}(x^k))$ (See Appendix~\ref{app:CO})} \citep{mescheder2017numerics}, the stochastic variant of \textbf{SGDA}, and finally the stochastic variance-reduced extragradient with restart \textbf{SVRE} proposed in \citep{chavdarova2019reducing}.
For all our experiments, we ran the different algorithms with 10 different seeds and plot the mean and 95\% confidence intervals.
We provide further details about the experiments and choice of hyperparameters for the different methods in Appendix~\ref{app:experiments-details}.

\subsection{Bilinear Games}
\label{exp:bilinear-games}

First we compare the different methods on the stochastic bilinear problem \eqref{bilinearGame1}. Similarly to \citet{chavdarova2019reducing}, we choose $n=d_1=d_2=100$, $[\bA_i]_{kl} = 1$ if $i = k = l$ and 0 otherwise, and $[b_i]_k, [c_i]_k \sim \mathcal{N}(0, 1/n)$.

We show the convergence of the different algorithms in Fig.~\ref{fig:bilinear-game}. As predicted by theory, \textbf{SHGD} with decreasing step-size converges at a sublinear rate while \textbf{L-SVRHG} converges at a linear rate. Among all the methods we compared to, \textbf{L-SVRHG} is the fastest to converge; however, the speed of convergence depends a lot on parameter $p$. We observe that setting $p=1/n$ yields the best performance.

To further illustrate the behavior of the Hamiltonian methods, 
we look at the trajectory of the methods on a simple 2D version of the bilinear game, where we choose $x_1$ and $x_2$ to be scalars. We observe that while previously proposed methods such as \textbf{SGDA} and \textbf{SVRE} suffer from rotations which slow down their convergence and can even make them diverge, the Hamiltonian methods converge much faster by removing rotation and converging ``straight" to the solution.

\subsection{Sufficiently-Bilinear Games}
\label{exp:nonlinear-games}
In section~\ref{sufficiently-bilinear-games}, we showed that Hamiltonian methods are also guaranteed to converge when the problem is non-convex non-concave but satisfies the sufficiently-bilinear condition~\eqref{SufficientBilinear}. 
To illustrate these results, we propose to look at the 
following game inspired by \citet{abernethy2019last}:
\begin{multline}
\label{eq:exp-sufficiently-bilinear}
\min_{x_1 \in\R^{d}}\max_{x_2 \in\R^{d}}\frac{1}{n} \sum_{i=1}^n  \big( F(x_1) + \delta \,\, x_1^\top \bA_i x_2 \, +  \\ 
 b_i^\top x_1 + c_i^\top x_2 - F(x_2) \big) ,
\end{multline}
where $F(x)$ is a non-linear function (see details in Appendix~\ref{app:nonlinear-games}).
This game is non-convex non-concave and satisfies the sufficiently-bilinear condition if $\delta > 2L$, where $L$ is the smoothness of $F(x)$. Thus, the results and theorems from Section~\ref{sufficiently-bilinear-games} hold. 

Results are shown in Fig.\ref{fig:nonlinear-game}. Similarly to the bilinear case, the methods follow very closely the theory.
We highlight that while the proposed theory for this setting only guarantees
convergence for \textbf{L-SVRHG} with restart, in practice using restart is not strictly necessary:  \textbf{L-SVRHG} with the correct choice of stepsize also converges in our experiment.
Finally we show the trajectories of the different methods on a 2D version of the problem. We observe that contrary to the bilinear case, stochastic SGDA converges but still suffers from rotation compared to Hamiltonian methods.
\begin{figure*}[h]
\captionsetup[subfigure]{justification=centering}
\begin{subfigure}[t]{0.5\textwidth}
   \begin{center}
    \hspace{-5mm}
   \includegraphics[width=.53\columnwidth]{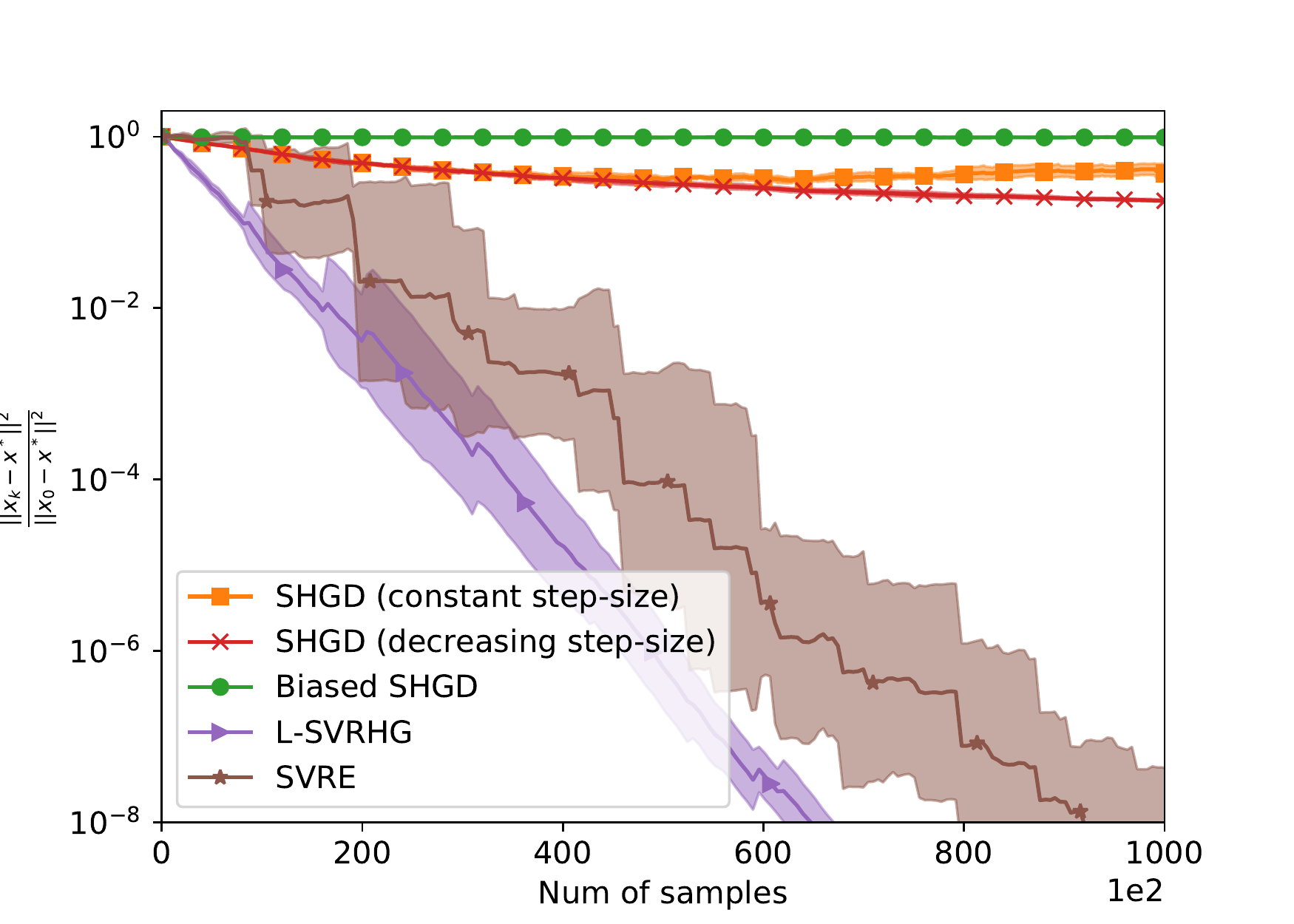}
   \hspace{-5mm}
   \includegraphics[width=.53\columnwidth]{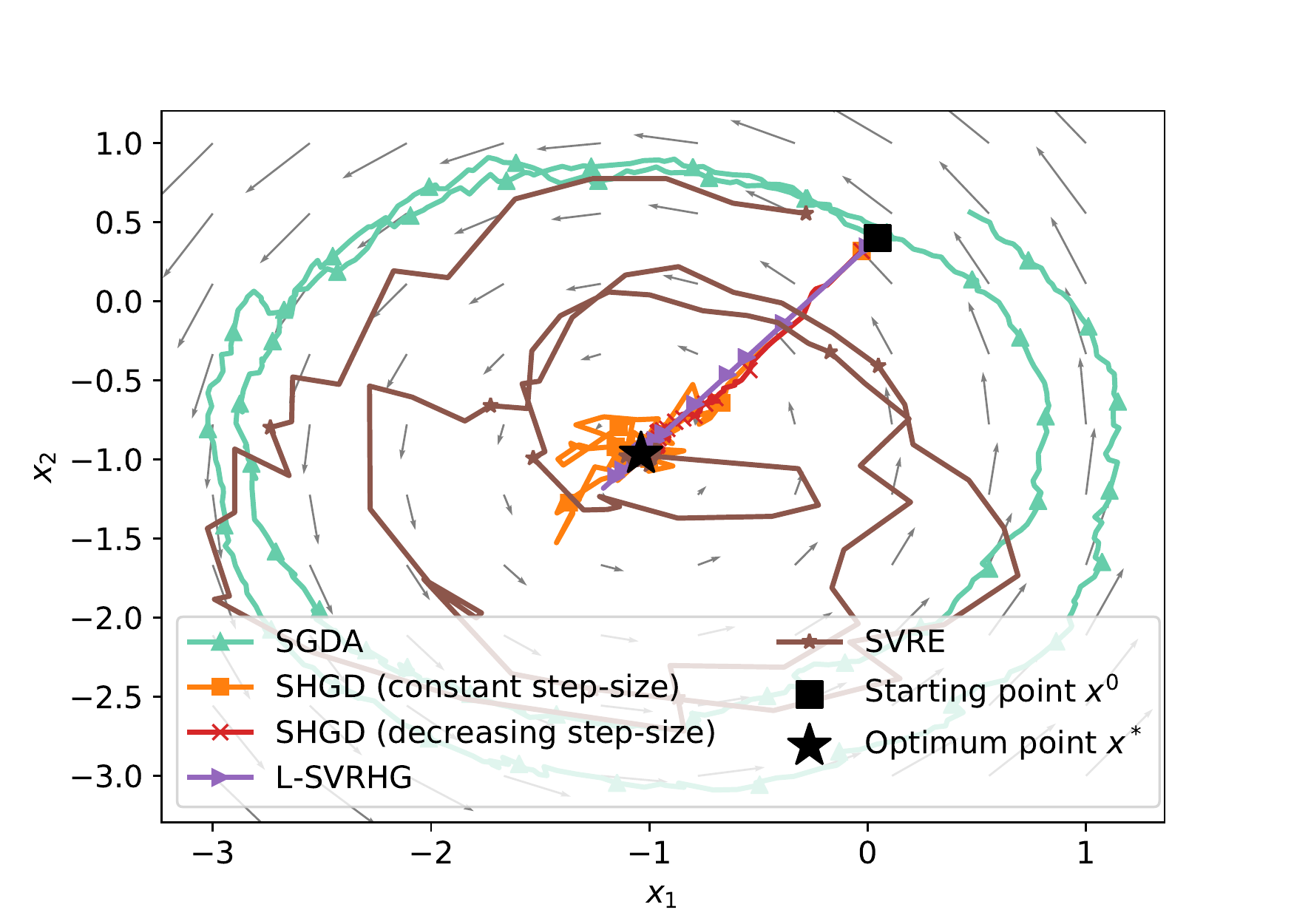}
   \caption{Bilinear game}
   \label{fig:bilinear-game}
   \end{center}
\end{subfigure}
\begin{subfigure}[t]{0.5\textwidth}
   \begin{center}
    \hspace{-4mm}
   \includegraphics[width=.53\columnwidth]{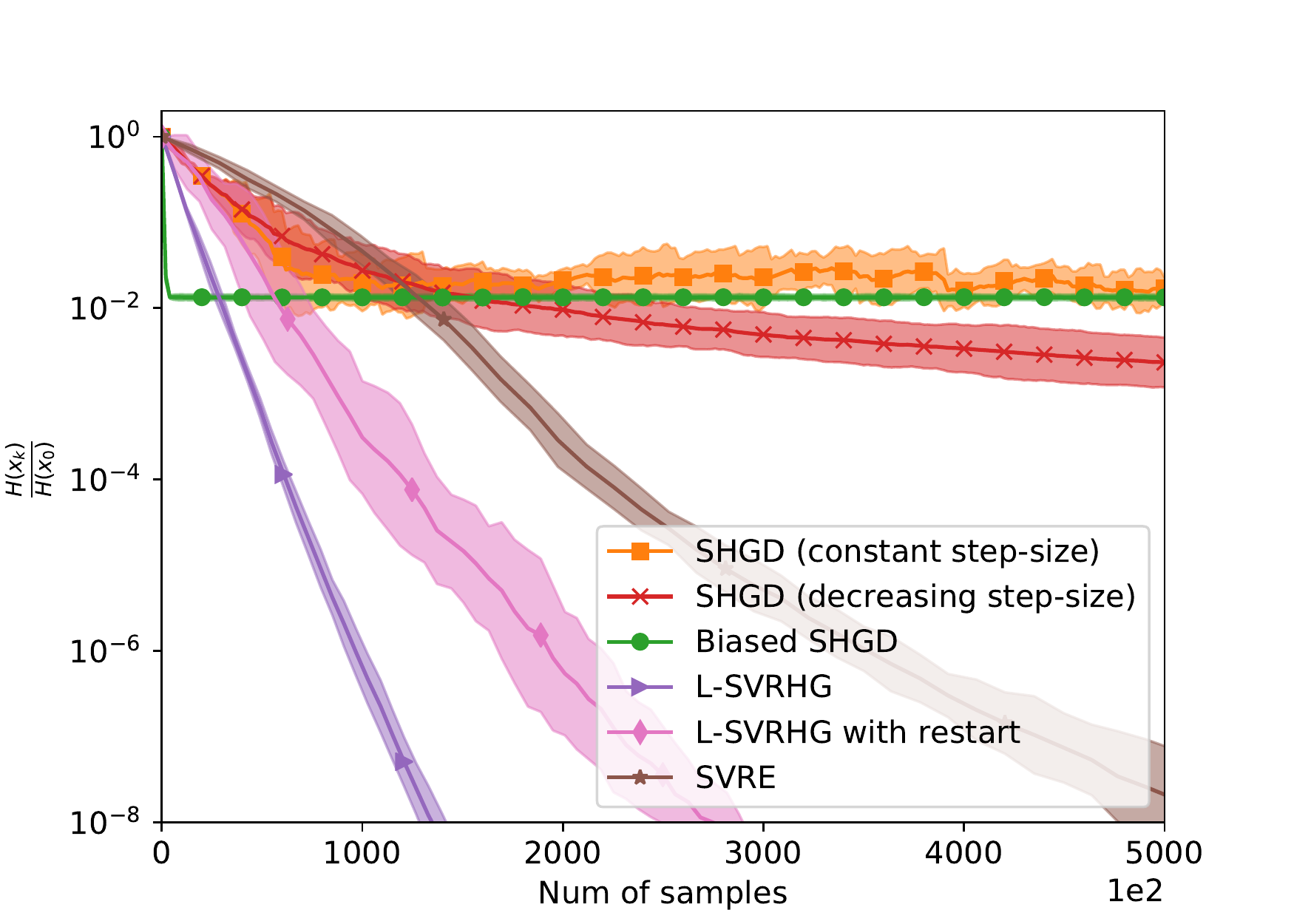}
   \hspace{-5mm}
   \includegraphics[width=.53\columnwidth]{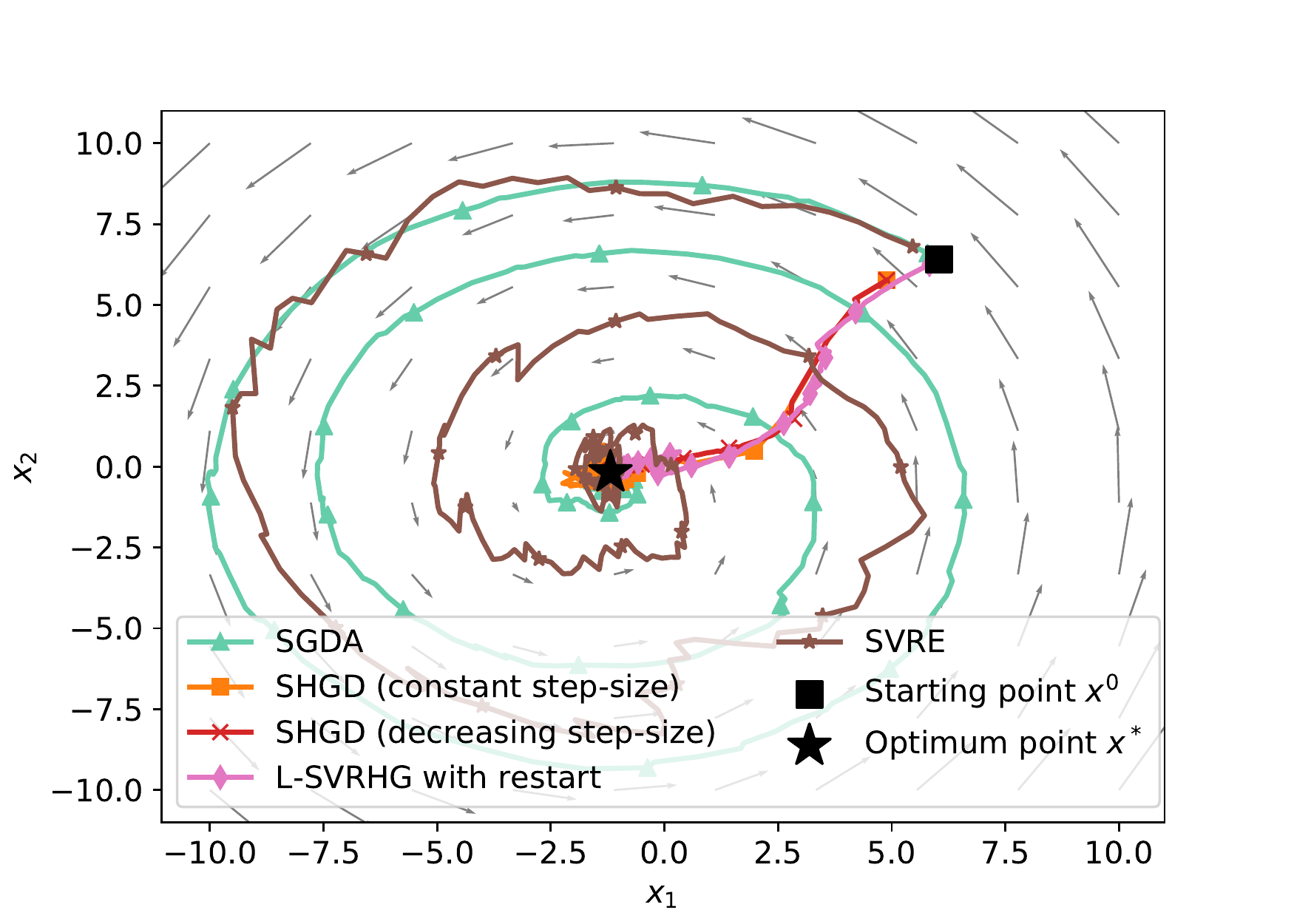}
   \caption{Sufficiently-bilinear game}
   \label{fig:nonlinear-game}
   \end{center}
\end{subfigure}
\vspace{-2mm}
\caption{\textbf{a)} Comparison of different methods on the stochastic bilinear game~\eqref{bilinearGame1}. Left: Distance to optimality $\frac{||x_k - x^*||^2}{||x_0 - x^*||^2}$ as a function of the number of samples seen during training. Right: The trajectory of the different methods on a 2D version of the problem.\\
\textbf{b)} Comparison of different methods on the sufficiently bilinear games~\eqref{eq:exp-sufficiently-bilinear}. Left: The Hamiltonian $\frac{H(x_k)}{H(x_0)}$ as a function of the number of samples seen during training. Right: The trajectory of the different methods on a 2D version of the problem.}
\vspace{-3mm}
\end{figure*}

\subsection{GANs}
\label{exp:GANs}
In previous experiments, we verify the proposed theory for the stochastic bilinear and sufficiently-bilinear games. Although we do not have theoretical results for more complex games, we wanted to test our algorithms on a simple GAN setting, which we call \em GaussianGAN\em.

In \em GaussianGAN\em, we have a dataset of real data $x_{real}$ and latent variable $z$ from a normal distribution with mean 0 and standard deviation 1. The generator is defined as $G(z)=\mu + \sigma z$ and the discriminator as $D(x_{data})=\phi_0 + \phi_1 x_{data} + \phi_2 x_{data}^2$, where $x_{data}$ is either real data ($x_{real}$) or fake generated data ($G(z)$). In this setting, the parameters are $x=(x_1,x_2)=([\mu, \sigma], [\phi_0,\phi_1,\phi_2])$.
In GaussianGAN, we can directly measure the $L^2$ distance between the generator's parameters and the true optimal parameters: $||\hat{\mu}-\mu|| + || \hat{\sigma}-\sigma ||$,
where $\hat{\mu}$ and $\hat{\sigma}$ are the sample's mean and standard deviation.

We consider three possible minmax games: Wasserstein GAN (WGAN) \citep{arjovsky2017wasserstein}, saturating GAN (satGAN) \citep{goodfellow2014generative}, and non-saturating GAN (nsGAN) \citep{goodfellow2014generative}. We present the results for WGAN and satGAN in Figure \ref{figgan1}. We provide the nsGAN results in Appendix~\ref{app:gans-other} and details for the different experiments in Appendix~\ref{app:gans}.
\begin{figure}[h]
\captionsetup[subfigure]{justification=centering}
\begin{subfigure}[t]{.23\textwidth}
   \begin{center}
   \centerline{\includegraphics[width=\columnwidth]{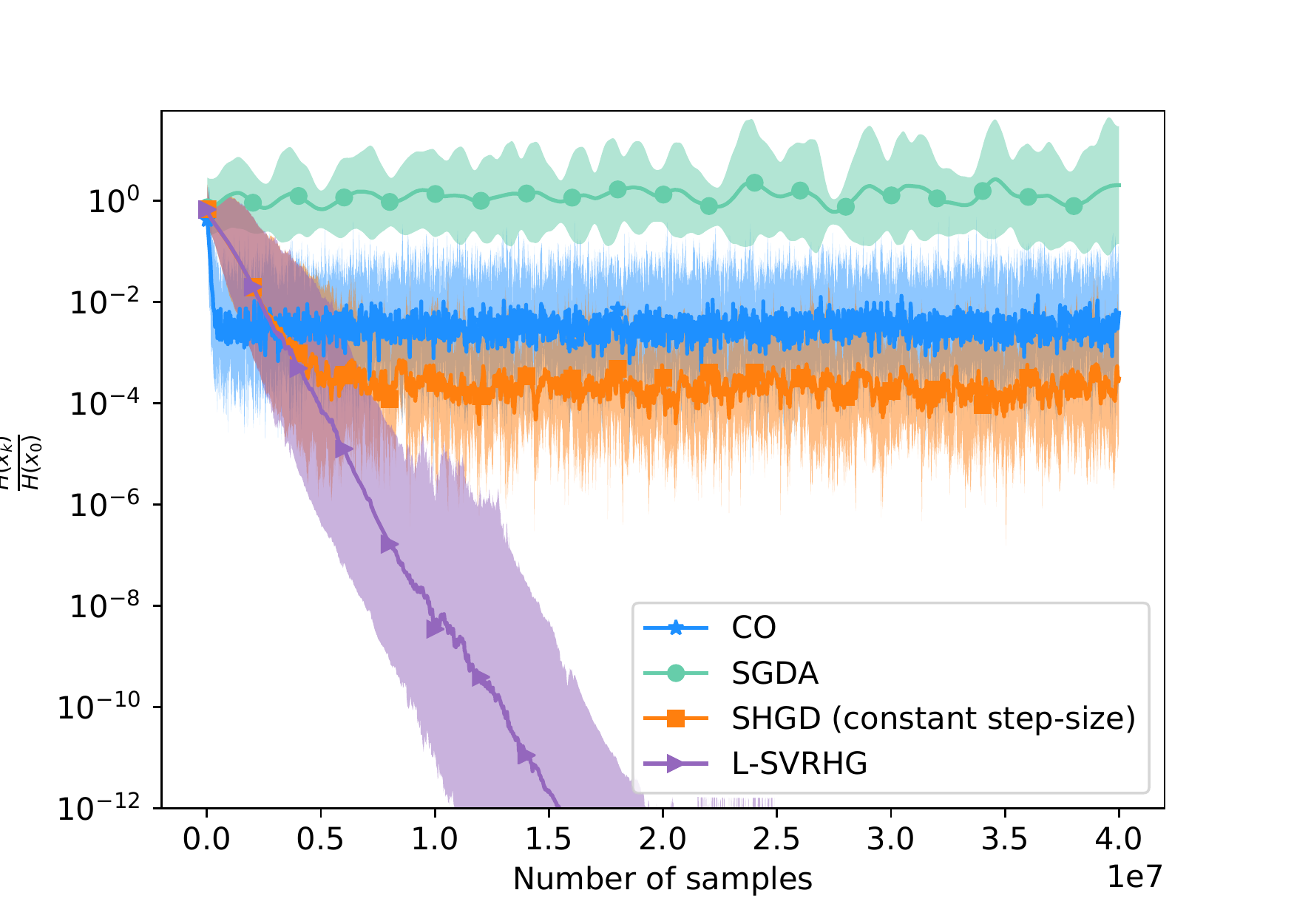}}
   \caption{Hamiltonian for WGAN}
   \end{center}
   \vspace{-6mm}
\end{subfigure}
\begin{subfigure}[t]{.23\textwidth}
   \begin{center}
   \centerline{\includegraphics[width=\columnwidth]{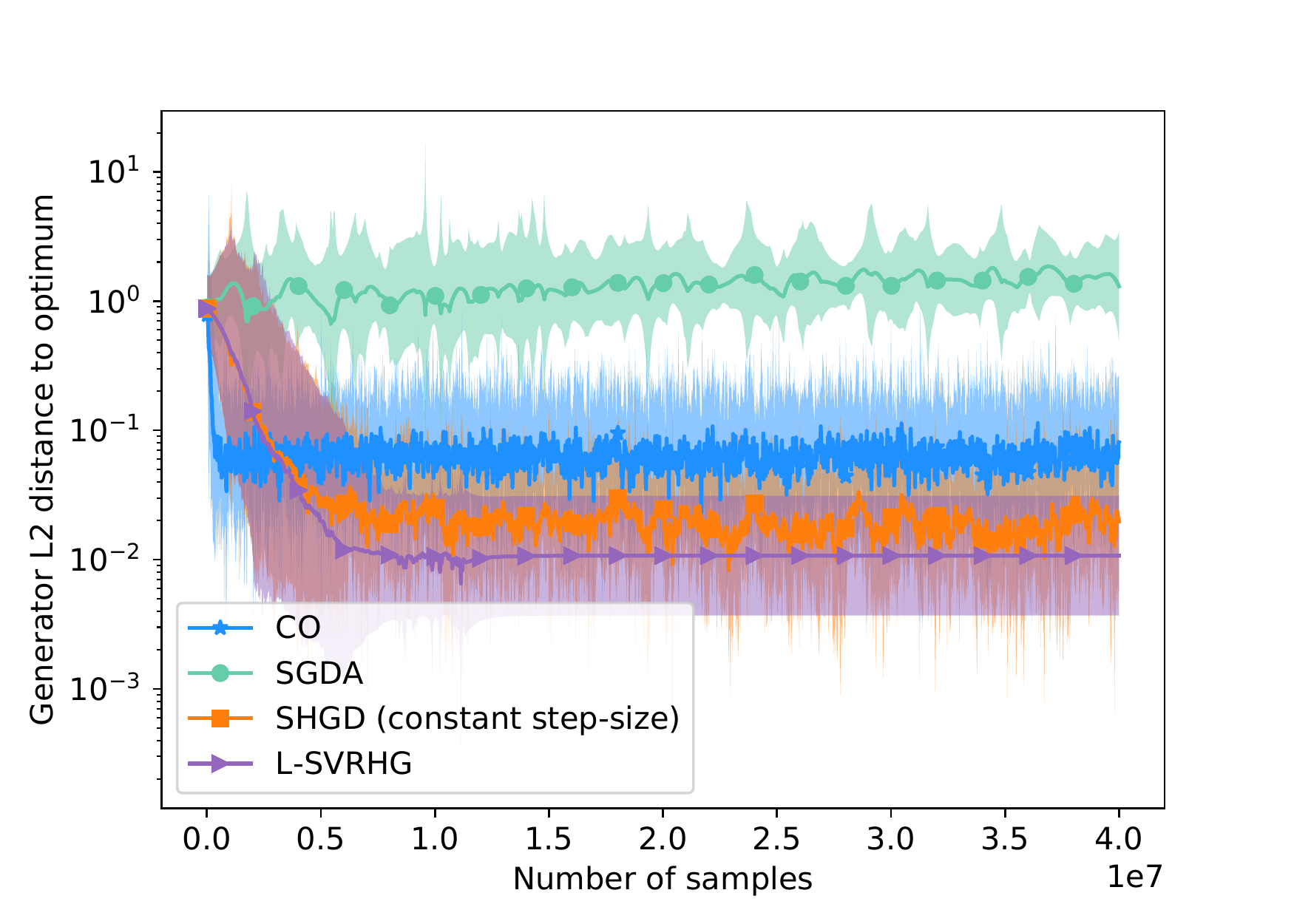}}
   \caption{Distance to optimum for WGAN}
   \end{center}
   \vspace{-6mm}
\end{subfigure}\\
\begin{subfigure}[t]{.23\textwidth}
   \begin{center}
   \centerline{\includegraphics[width=\columnwidth]{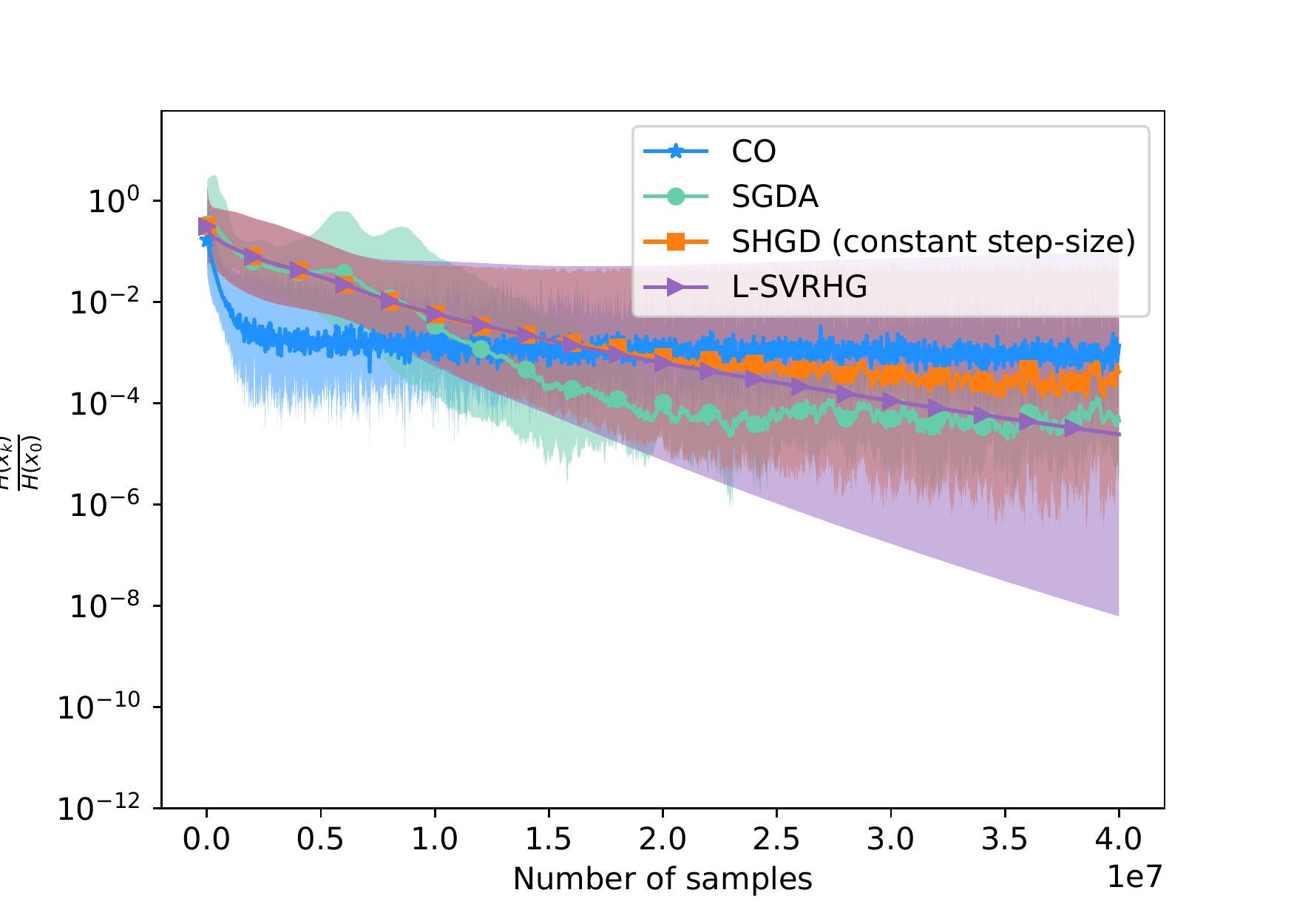}}
   \caption{Hamiltonian for satGAN}
   \end{center}
\end{subfigure}
\begin{subfigure}[t]{.23\textwidth}
   \begin{center}
   \centerline{\includegraphics[width=\columnwidth]{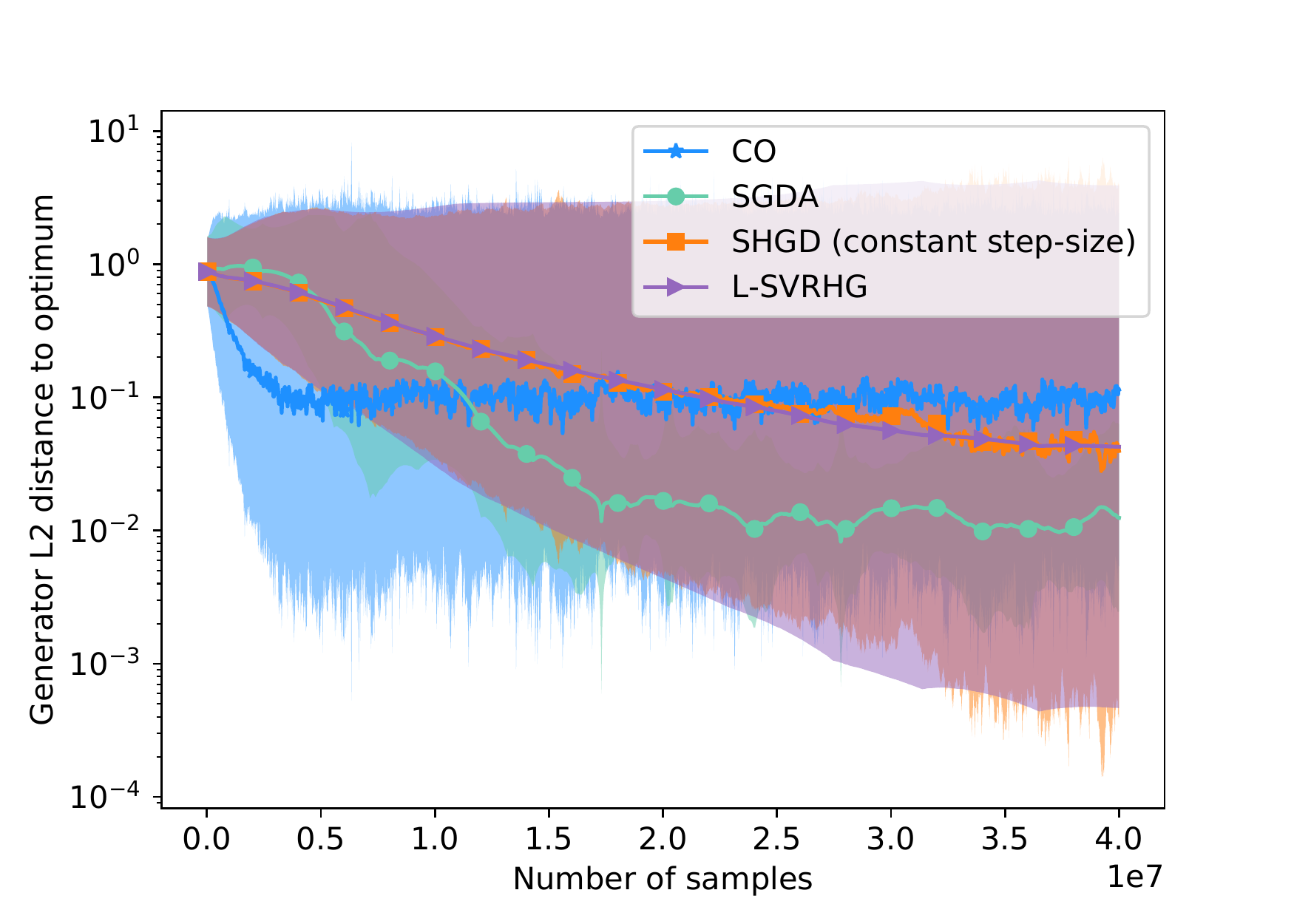}}
   \caption{Distance to optimum for satGAN}
   \end{center}
\end{subfigure}
\vspace{-4mm}
\caption{The Hamiltonian $\frac{H(x_k)}{H(x_0)}$ (\textbf{left}) and the distance to the optimal generator (\textbf{right}) as a function of the number of samples seen during training for WGAN and satGAN. The distance to the optimal generator corresponds to $\frac{||\hat{\mu}-\mu_k|| + ||\hat{\sigma}-\sigma_k||}{||\hat{\mu}-\mu_0|| + ||\hat{\sigma}-\sigma_0||}$.}
\label{figgan1}
\end{figure}

For WGAN, we see that stochastic \textbf{SGDA} fails to converge and that \textbf{L-SVRHG} is the only method to converge linearly on the Hamiltonian. 
For satGAN, \textbf{SGDA} seems to perform best. Algorithms that take into account the Hamiltonian have high variance. We looked at individual runs and found that, in 3 out of 10 runs, the algorithms other than stochastic \textbf{SGDA} fail to converge, and the Hamiltonian does not significantly decrease over time. While WGAN is guaranteed to have a unique critical point, which is the solution of the game, this is not the case for satGAN and nsGAN due to the non-linear component. Thus, as expected,  Assumption~\ref{criticalminmax} is very important in order for the proposed stochastic Hamiltonian methods to perform well.

\section{Conclusion and Extensions}
We introduce new variants of SHGD (through novel unbiased estimator and step-size selection) and present the first variance reduced Hamiltonian method L-SVRHG. Using tools from optimization literature, we provide convergence guarantees for the two methods and we show how they can efficiently solve stochastic unconstrained bilinear games and the more general class of  games that satisfy the ``sufficiently bilinear” condition. An important result of our analysis is the first set of global non-asymptotic last-iterate convergence guarantees for a stochastic game over a non-compact domain, in the absence of strong monotonicity assumptions.

We believe that our results and the Hamiltonian viewpoint could work as a first step in closing the gap between the stochastic optimization algorithms and methods for solving stochastic games and can open up many avenues for further development and research in both areas. 
A natural extension of our results will be the proposal of accelerated Hamiltonian methods that use momentum \cite{loizou2017momentum, assran2020convergence} on top of the Hamiltonian gradient update. We speculate that similar ideas to the ones presented in this work can be used for the development of efficient decentralized methods \cite{assran2018stochastic,koloskova2020unified} for solving problem~\eqref{MainStochasticProblem}.
\newpage
\section*{Acknowledgements}
The authors would like to thank Reyhane Askari, Gauthier Gidel and Lewis Liu for useful discussions and feedback.

Nicolas Loizou acknowledges support by the IVADO postdoctoral funding program.
This work was partially supported by the FRQNT new researcher program (2019-NC-257943), the NSERC Discovery grants (RGPIN-2017-06936 and RGPIN-2019-06512) and the Canada CIFAR AI chairs program.
Ioannis Mitliagkas acknowledges support by an IVADO startup grant and a Microsoft Research collaborative grant.
Simon Lacoste-Julien acknowledges support by a Google Focused Research award.
Simon Lacoste-Julien and Pascal Vincent are CIFAR Associate Fellows in the Learning in Machines \& Brains program.

\bibliography{StochasticHamiltonian}
\bibliographystyle{icml2020}

\appendix 

\onecolumn
\icmltitle{Supplementary Material \\ Stochastic Hamiltonian Gradient Methods for Smooth Games}

In the Appendix we present the proofs of the main Propositions and Theorems proposed in the main paper together with additional experiments on different Bilinear and sufficiently bilinear games. 

In particular in Section~\ref{AppendixTechnical}, we start by presenting the pseudo-codes of the stochastic optimization algorithms SGD and L-SVRG based on which we build our stochastic Hamiltonian methods.
In Section~\ref{Section2Appednix} we provide more details on the assumptions and definitions used in the main paper. In Section~\ref{ProofsMainPropositionsAppendix} we present the proofs of the two main propositions and in Section~\ref{sec:analysisAppendix} we explain how these propositions can be combined with existing convergence results in order to obtain the Theorems of Section~\ref{sec:analysis}. Finally in Sections~\ref{app:experiments-details} and \ref{AddExpAppendix} we present the experimental details and provide additional experiments.

\section{Stochastic Optimization Algorithms}
\label{AppendixTechnical}
In this section we present the pseudocodes of SGD and L-SVRG for solving the finite-sum optimization problem:
\begin{equation}
\label{FiniteSumOPT}
\min_{x\in\R^d} \left[ f(x) = \frac{1}{n} \sum_{i=1}^n f_i(x) \right].
\end{equation}

\begin{algorithm}[H]%[tb]
   \caption{Stochastic Gradient Descent (SGD)}
   \label{SGD_Algorithm}
\begin{algorithmic}
   \STATE {\bfseries Input:} Starting stepsize $\gamma^0>0$. Choose initial points $x^0 \in \R^d$. Distribution $\cD$ of samples.
   %\REPEAT
   \FOR{$k=0,2,\cdots, K$}
   \STATE Generate fresh sample $i \sim {\cal D}$ and evaluate $\nabla f_{i}(x^k)$.
   \STATE Set step-size $\gamma^k$ following one of the selected choices (constant, decreasing)
   \STATE Set $x^{k+1}=x^k -\gamma^k \nabla f_{i}(x^k)$
   \ENDFOR
   \STATE {\bf Output:} The last iterate $x^k$.
\end{algorithmic}
\end{algorithm}

\begin{algorithm}[H]%[tb]
   \caption{Loopless Stochastic Variance Reduced Gradient (L-SVRG)}
   \label{LSVRG_Algorithm}
\begin{algorithmic}
   \STATE {\bfseries Input:} Starting stepsize $\gamma>0$. Choose initial points $x^0=w^0 \in \R^d$. Distribution $\cD$ of samples. Probability $p \in (0,1]$.
   %\REPEAT
   \FOR{$k=0,2,\cdots, K$}
   \STATE Generate fresh sample $i \sim {\cal D}$ evaluate $\nabla f_{i}(x^k)$.
   \STATE Evaluate $g^k= \nabla f_{i}(x^k)-\nabla f_{i}(w^k)+\nabla f(w^k)$.
   \STATE Set $x^{k+1}=x^k -\gamma g^k$
   \STATE Set $$ w^{k+1} = \begin{cases} x^k \quad \text{with probability} \quad p\\  w^k \quad \text{with probability} \quad 1-p \end{cases}$$
   \ENDFOR
     \STATE {\bf Output:} \\
   Option I: The last iterate $x=x^k$.\\
   Option II: $x$ is chosen uniformly at random from $\{x^i\}^K_{i=0}$.
\end{algorithmic}
\end{algorithm}

\begin{algorithm}[tb]
   \caption{L-SVRG (with Restart)}
   \label{LSVRG_Restart}
\begin{algorithmic}
     \STATE {\bfseries Input:} Starting stepsize $\gamma>0$. Choose initial points $x^0=w^0 \in \R^d$. Distribution $\cD$ of samples. Probability $p \in (0,1]$, $T$
   %\REPEAT
   \FOR{$t=0,1,2,\cdots, T$}
   \STATE Set $x^{t+1} $ = L-SVRG$_{II}(x^{t}, K, \gamma, p \in (0,1]$)
   \ENDFOR
   \STATE {\bf Output:} The last iterate $x^T$.
\end{algorithmic}
\end{algorithm}

\section{Connections of Main Assumptions and Definitions}
\label{Section2Appednix}
As we mentioned above SGD (Algorithm~\ref{SGD_Algorithm}) and L-SVRG (Algorithm~\ref{LSVRG_Algorithm}) are popular methods for solving the stochastic optimization problem~\eqref{FiniteSumOPT}.  Several convergence analyses of the two algorithms have been proposed under different assumptions on the functions $f$ and $f_i$. In this section we describe in more details the assumptions used in the analysis of the stochastic Hamiltonian methods in the main paper.

\subsection{On Quasi-strong convexity and PL condition}
In Section~\ref{optBackground} we present the definitions of quasi-strong convexity and the PL condition and later in Section~\ref{stoHamilFunction} we explain that for the two classes of games (Bilinear and Sufficiently bilinear) the stochastic Hamiltonian function~\eqref{StochHamiltonianFunction} satisfies one of these conditions. Here using \citet{karimi2016linear} we explain the connection between these conditions and the more well-known definition of strong convexity.

\begin{definition}[Strong Convexity]
A differentiable function $f : \R^n \rightarrow \R$, is $\mu$-strongly convex, if there exists a constant
$\mu > 0$ such that $\forall x, y \in \R^n$:
\begin{equation}
\label{stronglyconvex}
 f(x) \geq f(y)+ \dotprod{\nabla f(y) , x-y} + \frac{\mu}{2} \norm{x-y}^2
\end{equation}
\end{definition}

In particular the following connection hold:
\begin{equation}
\label{cnoasao}
SC \subseteq QSC \subseteq PL, 
\end{equation}
where $SC$ denotes the class of strongly convex functions, $QSC$ the class of quasi-strongly convex (Definition~\ref{QSCdefinition}) and $PL$ the class of functions satisfy the PL condition (Definition~\ref{Polyak}). For more details on the connections of the $\mu$ parameter between these methods we refer the reader to \citet{karimi2016linear}  and \citet{Necoara-Nesterov-Glineur-2018-linear-without-strong-convexity}.

\subsection{On Smoothness and Expected Smoothness / Expected Residual}
\label{ESandER}
In Section~\ref{optBackground} we present the definitions of Expected Smoothness (ES) and Expected Residual (ER). In the main theoretical results of Section~\ref{sec:analysis} we also use the expected smoothness parameter $\cL$ and the expected residual parameter $\rho$ to provide the convergence guarantees of SHGD and L-SVRHG. 
In this section we provide more details on these assumptions as presented in \citet{gower2019sgd,gower2020sgd}.

As explained in \citet{gower2019sgd,gower2020sgd} expected smoothness and expected residual are assumptions that combine both the properties of the distribution $\cD$ of drawing samples and the smoothness properties of function $f$. In particular, ES and ER can be seen as two different ways to measure how far the gradient estimate $\nabla f_i(x)$ is from the true gradient $\nabla f(x)$ where $i\sim\cD$. 

ES was first used for the analysis of SGD in \citet{gower2019sgd} for solving stochastic optimization problems of the form~\eqref{FiniteSumOPT} where the objective function $f$ is assumed to be $\mu$--quasi-strongly convex (see Definition~\ref{QSCdefinition}). Later in \citet{gower2020sgd} a similar analysis for SGD has been proposed for functions satisfying the PL condition. As explained in \citet{gower2020sgd}, assuming ES in the analysis of SGD for functions satisfying the PL condition is not ideal as it does not allow the recovery of the the best known dependence on the condition number for the deterministic Gradient Descent (full batch). For this reason \citet{gower2020sgd} used the notion of Expected residual (ER) in the proposed analysis and explained its benefits.

In both, \citet{gower2019sgd} and \cite{gower2020sgd},  the ES and ER assumptions have been used in combination with the arbitrary sampling paradigm. That is, the proposed theorems of~\citet{gower2019sgd,gower2020sgd} that describe the convergence of SGD include an infinite array of variants of SGD as special cases. Each one of these variants is associated with a specific probability law governing the data selection rule used to form minibatches.

\subsubsection{Formal Definitions} Let us present the definitions of ES and ER as presented in \citet{gower2019sgd} and \citet{gower2020sgd}. In \citet{gower2019sgd, gower2020sgd} to allow for any form of minibatching the \emph{arbitrary sampling} notation was used. That is,
\begin{equation}
\nabla f_v(x) \eqdef \tfrac{1}{n} \sum _{i=1}^n v_i \nabla f_i(x),
\end{equation}
where $v\in\R^n_+$ is a random \emph{sampling vector} such that $\E{v_i}  = 1, \, \mbox{for }i=1,\ldots, n$ and $f_v(x) \eqdef \tfrac{1}{n}\sum_{i=1}^n v_i f_i(x)$.
Note that it follows immediately from this definition of sampling vector that $\E{\nabla f_v(x)} =\tfrac{1}{n} \sum _{i=1}^n \E{v_i} \nabla f_i(x) = \nabla f(x).$ 

In addition note that using the notion of \emph{arbitrary sampling} the update rule of SGD is simply: $x^{k+1}=x^k -\gamma^k \nabla f_v(x^k)$.

Under the notion of \emph{arbitrary sampling} the expected smoothness assumption \cite{gower2019sgd} and the expected residual assumption \cite{gower2020sgd} take the following form (generalization of the definitions presented in Section~\ref{optBackground}).

\begin{assumption}[Expected Smoothness (ES)]
We say that $f$ is $\cL$--smooth in expectation with respect to a distribution $\cD$ if there exists  $\cL=\cL(f,\cD)>0$  such that
\begin{equation}
\EE{\cD}{\norm{\nabla f_v(x)-\nabla f_v(x^*)}^2} \leq 2\cL (f(x)-f(x^*)),
\end{equation}
for all $x\in\R^d$. For simplicity, we will write $(f,\cD)\sim ES(\cL)$ to say that expected smoothness holds. 
\end{assumption}
\begin{assumption}[Expected Residual (ER)]
We say that $f$ satisfied the expected residual assumption if there exists  $\rho=\rho(f,\cD)>0$ such that
\begin{equation}
\EE{\cD}{\norm{\nabla f_v(x)-\nabla f_v(x^*) -  ( \nabla f(x)-\nabla f(x^*))}^2}  \leq 2\rho\left(f(x)-f(x^*) \right).
\end{equation}
for all $x\in\R^d$. For simplicity, we will write $(f,\cD) \sim ER(\rho)$ to say that expected residual holds. 
\end{assumption}

As we explain in Section~\ref{sec:analysis}, in this work we focus on $\tau$-minibatch sampling, where in each step we select uniformly at random a minibatch of size $\tau \in [n^2]$ (recall that the Hamiltonian function~\eqref{StochHamiltonianFunction} has $n^2$ components). However we highlight that the proposed analysis of the stochastic Hamiltonian methods holds for any form of sampling vector following the results presented in \citet{gower2019sgd,gower2020sgd} for the case of SGD and \citet{qian2019svrg} for the case of L-SVRG methods, including importance sampling variants.

Let us provide a formal definition of the $\tau$-minibatch sampling when $\tau \in [n]$.
\begin{definition}[$\tau$-Minibatch sampling]\label{def:minibatch}
Let $\tau \in [n]$. We say that $v \in \R^n$ is a $\tau$--minibatch sampling if
for every subset $S \in [n]$ with $|S| =\tau$ we have that
$\Prob{v=\tfrac{n}{\tau}\sum_{i \in S} e_i}=1/\binom{n}{\tau} \eqdef \tau!(n-\tau)!/n!$
\end{definition}
It is easy to verify by using a double counting argument that if $v$ is a $\tau$--minibatch sampling, it is also a valid sampling vector ($\E{v_i}  = 1$)~\cite{gower2019sgd}. 

Let $f(x) =\frac{1}{n} \sum_{i=1}^n f_i(x)$ with functions
$f_i$ be $L_{i}$--smooth and function $f$ be $L$-smooth and let $L_{\max}=\max_{\{1,\dots,n\}} \{L_i\}$. In this setting as it was shown in \citet{gower2019sgd,gower2020sgd} for the case of $\tau$-minibatch sampling ($\tau \in [n]$), the expected smoothness and expected residual parameters and the finite gradient noise $\sigma^2$ take the following form:
\label{prop:bniceconst} 
\begin{gather} 
\label{nic11}
 \cL(\tau) = \frac{n(\tau-1)}{\tau(n-1)}L + \frac{n-\tau}{\tau(n-1)}L_{\max}\\ 
 \label{nic22}
\rho(\tau)=L_{\max} \frac{n-\tau}{(n-1)\tau}\\
 \label{nic33}
\sigma^2(\tau) \eqdef \Exp_{\cD}[\norm{\nabla f_v(x^*)}^2]=\frac{1}{\tau} \frac{n-\tau}{n-1} \frac{1}{n} \sum_{i=1}^n \norm{\nabla f_i(x^*)}^2.
\end{gather}

Using the above expressions \eqref{nic11}, \eqref{nic22} and \eqref{nic33} it is easy to see that for single element sampling where $\tau=1$ it holds that $\cL=\rho=L_{\max}$ and that $\sigma^2=\frac{1}{n} \sum_{i=1}^n \norm{\nabla f_i(x^*)}^2$. 
On the other limit case where a full-batch is used ($\tau=n$), these values become $\cL=L$ and $\rho=\sigma^2=0$.
Note that these are exactly the values for $\cL$, $\rho$ and $\sigma^2$ we use in Section~\ref{sec:analysis} with the only difference that $\tau \in [n^2]$ because the stochastic Hamiltonian function~\eqref{StochHamiltonianFunction} has $n^2$ components $\cH_{i,j}$.

In particular, as we explained in Section~\ref{sec:analysis}, for the Theorems related to SHGD we use $\sigma^2  \eqdef \Exp_{i,j}[\norm{\nabla \cH_{i,j}(x^*)}^2]$. From the above expression and for the case of $\tau$-minibatch sampling with $\tau \in [n^2]$ this is equivalent to:
$$\sigma^2  \eqdef \Exp_{i,j}[\norm{\nabla \cH_{i,j}(x^*)}^2]=\frac{1}{\tau} \frac{n^2-\tau}{n^2-1} \frac{1}{n^2} \sum_{i=1}^n  \sum_{j=1}^n \norm{\nabla \cH_{i,j}(x^*)}^2.$$

\paragraph{Connection between $\tau$-minibatch sampling and sampling step of main algorithms.}
Note that one of the main steps of Algorithms~\ref{SHGD_Algorithm} and \ref{LSVRG_Algorithm} is the generation of fresh samples $i \sim {\cal D}$ and $j \sim {\cal D}$ and the evaluation of $\nabla \cH_{i,j}(x^k)$. In the case of uniform single element sampling, the samples $i$ and $j$ are selected with probability $p_i=1/n$ and  $p_j=1/n$ respectively. This is equivalent on selecting samples $\{i,j\}$ uniformly at random from the $n^2$ components of the Hamiltonian function.  In both cases the probability of selecting the component $\cH_{i,j}$ is equal to $p_{\cH_{i,j}}=1/n^2$. 

In other words, for the case of $1$-minibatch sampling (uniform single element sampling), one can simply substitute the sampling step of SHGD and L-SVRHG:  ``\textit{Generate fresh samples $i \sim {\cal D}$ and  $j \sim {\cal D}$ and evaluate $\nabla \cH_{i,j}(x^k)$.}" with the ``\textit{Sample uniformly at random the component $H_{i,j}$ and evaluate $\nabla \cH_{i,j}(x^k)$.}"

Trivially, using the definition~\eqref{def:minibatch} and the above notion of \emph{sampling vector}, this connection can be extended to capture the more general $\tau$-minibatch sampling where $\tau \in [n^2]$. In this case, we will have
$$\nabla \cH_{v}(x) \eqdef \tfrac{1}{n^2} \sum _{i=1}^n \sum_{j=1}^n v_{i,j} \nabla \cH_{i,j}(x),$$
where $v\in\R^{n^2}_+$ is a random \emph{sampling vector} such that $\Exp_{i,j}[v_{i,j}]  = 1, \, \mbox{for }i=1,\ldots, n, \, \mbox{and } j=1,\ldots, n$ and $\cH_{v}(x) \eqdef \tfrac{1}{n^2} \sum _{i=1}^n \sum_{j=1}^n v_{i,j} \cH_{i,j}(x)$.
Note that it follows immediately from this definition of sampling vector that $\E{\nabla \cH_{v}(x)} =\tfrac{1}{n^2}  \sum _{i=1}^n \sum_{j=1}^n \Exp_{i,j}[v_{i,j}]  \nabla \cH_{i,j}(x) = \nabla \cH(x).$ 

In this case the update rule of SHGD (Algorithm~\ref{SHGD_Algorithm}) will simply be: $x^{k+1}=x^k -\gamma^k \nabla \cH_{v}(x)$ and the proposed theoretical results will still hold.  

\subsubsection{Sufficient conditions and connections between notions of smoothness.} In \citet{gower2019sgd} it was proved that convexity and $L_i$--smoothness of $f_i$ in problem~ \eqref{FiniteSumOPT} implies expected smoothness of function $f$. However, the opposite implication does not hold. The expected smoothness assumption can hold even when the $f_i$'s and $f$ are not convex. See \cite{gower2019sgd} for more details. 

Similar results have been shown in \citet{gower2020sgd} for the case of expected residual. More specifically, it was shown that if the functions $f_i$ of problem \eqref{FiniteSumOPT} are $L_i$--smooth and also $x^*$-convex for $x^* \in \cX^*$ (where $\cX^*$ is solution set of $f$) then function $f$ satisfies the expected residual conditions, that is $(f,\cD) \sim ER(\rho)$ and the expected residual parameter $\rho$ has a meaningful expression.

Another interesting connections between the smoothness parameters is the following \citet{gower2019sgd}. 
If we assume that function $f$ of problem~\eqref{FiniteSumOPT} is $L$–smooth and that each $f_i$ function is $L_i$–smooth then the expected smoothness $\cL$ constant is bounded as follows:
$$L \leq \cL \leq L_{\max},$$
where $L_{\max}=\max{L_i}_{i=1}^n$.

Let us also present the following lemma as proved in \cite{gower2020sgd} that connects the ES and ER assumptions.
\begin{lemma}(Expected smoothness implies Expected Residual, from \cite{gower2020sgd}.)
If function $f$ of problem~\eqref{FiniteSumOPT}  satisfies the expected smoothness $(f,\cD)\sim ES(\cL)$, then it satisfies the expected residual $(f,\cD)\sim ER(\rho)$ with $\rho=\cL$. If in addition the function satisfied the PL condition that satisfied the expected residual with $\rho=\cL-\mu.$
\end{lemma}

\subsubsection{Bounds on the Stochastic Gradient} 
A common assumption used to prove the convergence of SGD is uniform boundedness of the stochastic gradients: there exist $0<c < \infty$ such that $\Exp\|\nabla f_{v} (x)\|^2 \leq c$ for all $x$. However, this assumption often does not hold, such as in the case when $f$ is strongly convex~\cite{pmlr-v80-nguyen18c}. Recall that the class of $\mu$-strongly convex functions is a special case of both the $\mu$-quasi strongly convex and functions satisfying the PL condition (see \eqref{cnoasao}).

Using ES and ER in the proposed theorems we do not need to assume such a bound. Instead, we use the following direct consequence of expected smoothness and expected residual to bound the expected norm of the stochastic gradients.
\begin{lemma}\cite{gower2019sgd}
\label{lem:weakgrowth}
If $(f,\cD)\sim ES(\cL)$, then
\begin{align}
\label{upperbound}
\Exp_{\cD} \left[ \|\nabla f_{v} (x)\|^2 \right] & \leq  4  \cL ( f(x)-f(x^*) ) + 2 \sigma^2,
\end{align}
where $\sigma^2  \eqdef \Exp_{\cD}[\norm{\nabla f_v(x^*)}^2]$.
\end{lemma}
Similar upper bound on the stochastic gradients can be obtained if one assumed expected residual:
\begin{lemma}\cite{gower2020sgd}
\label{lem:varbndrho}
If $(f,\cD) \sim ER(\rho)$ then 
\begin{align}
\label{eq:varbndrho2}%\tag{BG}
\Exp_{\cD} \left[ \|\nabla f_{v} (x)\|^2 \right]  & \leq  4  \rho ( f(x)-f^* ) + \|\nabla f (x)\|^2   +2 \sigma^2.
\end{align}
where $\sigma^2  \eqdef \Exp_{\cD}[\norm{\nabla f_v(x^*)}^2]$.
\end{lemma}

\section{On Stochastic Hamiltonian Function and Unbiased Estimator of the Gradient}
\subsection{Finite-Sum Structure of Hamiltonian Methods}
\label{OnStructureHamiltonian}
Having $g$ to be a finite sum function leads to the following derivations on the Hamiltonian functions and gradients:
\begin{eqnarray}
\label{nalksxalx}
\cH(x)&=&\frac{1}{2} \|\xi(x)\|^2= \frac{1}{2} \langle \xi(x), \xi(x) \rangle =\frac{1}{2} \langle \frac{1}{n}\sum_{i=1}^n \xi_i(x), \frac{1}{n}\sum_{j=1}^n \xi_j(x)\rangle \notag\\
&=&\frac{1}{n^2} \sum_{i=1}^n \sum_{j=1}^n \underbrace{\frac{1}{2} \langle  \xi_i(x),  \xi_j(x)\rangle}_{\cH_{i,j}(x)}=\frac{1}{n^2} \sum_{i,j=1}^n \underbrace{\frac{1}{2} \langle  \xi_i(x),  \xi_j(x)\rangle}_{\cH_{i,j}(x)}
\end{eqnarray}
That is, the Hamiltonian function $\cH(x)$ can be expressed as a finite sum with $n^2$ components.

\subsection{Unbiased Estimator of the Full Gradient}
The gradient of $\cH_{i,j}(x)$ has the following form:
\begin{eqnarray}
\label{nalksxalxaslck}
\nabla \cH_{i,j}(x)&=& \frac{1}{2} \nabla \langle  \xi_i(x),  \xi_j(x)\rangle = \frac{1}{2} \left[  \langle \nabla \xi_i(x),  \xi_j(x)\rangle +   \langle  \xi_i(x), \nabla \xi_j(x)\rangle  \right]=\frac{1}{2} \left[ \bJ_i^\top \xi_j +   \bJ_j^\top \xi_i  \right],
\end{eqnarray}
and it is an unbiased estimator of the full gradient. That is,  $\nabla \cH(x)=\Exp_{i,j} [\nabla \cH_{i,j}(x)]$.
\begin{eqnarray}
\label{nalksxalx2}
\nabla \cH(x)&=& \nabla \frac{1}{2} \|\xi(x)\|^2= \nabla \frac{1}{2} \langle \xi(x), \xi(x) \rangle = \frac{1}{n^2} \sum_{i=1}^n \sum_{j=1}^n \frac{1}{2}\nabla \langle  \xi_i(x),  \xi_j(x)\rangle\notag\\
&=&\frac{1}{n^2} \sum_{i=1}^n \sum_{j=1}^n \frac{1}{2}\left[  \langle \nabla \xi_i(x),  \xi_j(x)\rangle +   \langle  \xi_i(x), \nabla \xi_j(x)\rangle  \right]\notag\\
&=&\frac{1}{n^2} \sum_{i=1}^n \sum_{j=1}^n \underbrace{\frac{1}{2}\left[ \bJ_i^\top \xi_j +   \bJ_j^\top \xi_i  \right]}_{\nabla \cH_{i,j}(x)}\notag\\
&=&\frac{1}{n} \sum_{i=1}^n \frac{1}{n}\sum_{j=1}^n \nabla \cH_{i,j}(x)\notag\\
&=&\Exp_{i} \Exp_{j} [\nabla \cH_{i,j}(x)]=\Exp_{i,j} [\nabla \cH_{i,j}(x)]
\end{eqnarray}

\subsection{Beyond Finite-Sum}
\label{BeyondFiniteSum}
All results presented in the main paper related to SHGD can be trivially extended beyond finite sum problems (with exactly the same rates). The finite-sum structure is required only for the variance reduced method L-SVRHG. In the stochastic case, problem (3) will be $$\min_{x_1 \in\R^{d_1}} \max_{x_2 \in\R^{d_2}} g(x_1, x_2) = E_\zeta [g (x,\zeta)]$$ where $\zeta$ is a random variable obeying some distribution. Then $\xi(x)=E_\zeta[\xi(x,\zeta)]$, $\bJ=E_\zeta[\bJ(x,\zeta)]$ and the stochastic Hamiltonian function will become 
$$H(x)=E_{\zeta_i} E_{\zeta_j} \underbrace{\frac{1}{2} \langle\xi(x,\zeta_i), \xi(x,\zeta_j)\rangle}_{\cH_{i,j}(x)}.$$

In this case $\nabla \cH_{i,j}(x)=\frac{1}{2} \left[ \bJ(x,\zeta_i)^\top \xi(x,\zeta_j) +   \bJ(x,\zeta_j)^\top \xi(x,\zeta_i)  \right]$ and $\nabla \cH(x)=E_{\zeta_i} E_{\zeta_j}  [\nabla \cH_{i,j}(x)]$.

In this case SHGD will execute the following updates in each step $k \in \{0,1,2,\cdots, K\}$:
\begin{enumerate}
\item Generate i.i.d random variables $\zeta_i$ and $\zeta_j$ and evaluate $\nabla \cH_{i,j}(x^k)$.
\item Set step-size $\gamma^k$ following one of the selected choices (constant, decreasing)
\item Set $x^{k+1}=x^k -\gamma^k \nabla \cH_{i,j}(x^k)$
\end{enumerate}

\section{Proofs of Main Propositions}
\label{ProofsMainPropositionsAppendix}
Let us first present the main notation used for eigenvalues and singular values (similar to the main paper).
\paragraph{Eigenvalues, singular values}
Let $\bA \in \R^{n\times n}$.We denote with $\lambda_1\leq \lambda_2 \leq \cdots \leq \lambda_{n}$ its  eigenvalues. Let $\lambda_{\min} = \lambda_1$ be the smallest non-zero eigenvalue, and $\lambda_{\max}  = \lambda_n$ be the largest eigenvalue. With $\sigma_1\leq \sigma_2 \leq \cdots \leq \sigma_{n}$ we denote its singular values. With $\sigma_{\max}$ and $\sigma_{\min}$ we denote the maximum singular value and the minimum non-zero singular value of matrix $\bA$.
 
\subsection{Proof of Proposition~\ref{BilinearGameProposition}}
In our proof we use the following result proved in \citet{Necoara-Nesterov-Glineur-2018-linear-without-strong-convexity}.
\begin{lemma}[\citet{Necoara-Nesterov-Glineur-2018-linear-without-strong-convexity}]
\label{Necoara}
Let function $z: \R^m \rightarrow \R$ be $\mu_z$-strongly convex with $L_z$-Lipschitz continuous gradient and 
$\bA \in \R^{m \times n}$ be a nonzero matrix. Then, the convex function
$f(x) = z(Ax)$ is a smooth $\mu$--quasi-strongly  convex function with constants $L= L_z\|\bA\|^2$ and $\mu =\mu_z \sigma_{\min}^2(\bA)$ where $\sigma_{\min}(\bA)$ is the smallest nonzero singular value of matrix $ \bA$ and $\|\bA\|$ is the spectral norm.
\end{lemma}

\begin{proof}
Recall that in the stochastic bilinear game we have that:
 $$g(x_1,x_2)=\frac{1}{n} \sum_{i=1}^n  \underbrace{x_1^\top b_i+x_1^\top \bA_i x_2 +c_i^\top x_2}_{g_i(x_1,x_2)}$$
 
 By evaluating the partial derivatives we obtain:
 \begin{itemize}
 \item $\nabla_{x_1}g_i(x_1,x_2)=\bA_i x_2 +b_i \in \R^{d_1 \times 1}$, $\forall i \in [n]$
 \item $\nabla_{x_2}g_i(x_1,x_2)=\bA_i^\top x_1 +c_i \in \R^{d_2 \times 1}$, $\forall i \in [n]$
 \end{itemize}
 
 Thus, from definition of $\xi_i(x)$ we get:
 $$\xi_i(x)=\left(\nabla_{x_1} g_i, -\nabla_{x_2} g_i \right)=\left(\bA_i x_2 +b_i , -[\bA_i^\top x_1 +c_i] \right) \quad \forall i \in [n]$$
 
and as a result by simple computations:
\begin{eqnarray}
\cH_{i,j}(x)&=& \frac{1}{2} \langle  \xi_i(x),  \xi_j(x)\rangle \notag\\
&=& \frac{1}{2} \left\langle  \left(\bA_i x_2 +b_i , -[\bA_i^\top x_1 +c_i] \right),  \left(\bA_j x_2 +b_j , -[\bA_j^\top x_1 +c_j] \right)\right\rangle \notag\\
&=& \frac{1}{2}(x_1,x_2)^\top
\begin{pmatrix} 
\bA_i \bA_j^\top   & 0\\ 
0 & \bA_i^\top  \bA_j\\ 
\end{pmatrix}
\begin{pmatrix} 
x_1\\ 
x_2\\ 
\end{pmatrix} +\frac{1}{2} \left(c_j^\top \bA_i^\top+c_i^\top \bA_j^\top, b_j^\top \bA_i+b_i^\top \bA_j \right) \begin{pmatrix} 
x_1\\ 
x_2\\ 
\end{pmatrix} +\frac{1}{2}c_i^\top c_j+\frac{1}{2}b_i^\top b_j \notag\\
&=& \frac{1}{2} x^\top \bQ_{i,j} x +q_{i,j}^\top x +\ell_{i,j},
\end{eqnarray}
where $\bQ_{i,j}=\begin{pmatrix} 
\bA_i \bA_j^\top   & 0\\ 
0 & \bA_i^\top  \bA_j\\ 
\end{pmatrix}$ and 

$q_{i,j}^\top= \frac{1}{2} \left(c_j^\top \bA_i^\top+c_i^\top \bA_j^\top, b_j^\top \bA_i+b_i^\top \bA_j \right)$ and $\ell_{i,j} =\frac{1}{2}c_i^\top c_j+\frac{1}{2}b_i^\top b_j $.

 Using the finite-sum structure of the Hamiltonian function \eqref{StochHamiltonianFunction} the stochastic Hamiltonian function takes the following form:
\begin{eqnarray}
\cH(x)&=&\frac{1}{n^2} \sum_{i,j=1}^n \cH_{i,j}(x)\notag\\
&=&\frac{1}{n^2} \sum_{i,j=1}^n \frac{1}{2} x^\top \bQ_{i,j} x +q_{i,j}^\top x +\ell_{i,j}\notag\\
&=& \frac{1}{2} x^\top \bQ x +q^\top x +\ell
\end{eqnarray}

where $\bQ= \left[\frac{1}{n^2} \sum_{i,j=1}^n\bQ_{i,j}\right]=\begin{pmatrix} 
\bA \bA^\top   & 0\\ 
0 & \bA^\top  \bA\\ 
\end{pmatrix}$ with $\bA=\frac{1}{n}\sum_{i=1}^n \bA_i$ and 

$q^\top=\left[\frac{1}{n^2} \sum_{i,j=1}^n q_{i,j}^\top\right]$ and $\ell=\left[\frac{1}{n^2} \sum_{i,j=1}^n \ell_{i,j}\right]$.

Thus, the stochastic Hamiltonian function $$\boxed{\cH(x)=\frac{1}{2} x^\top \bQ x +q^\top x +\ell}$$ is a quadratic smooth function (not necessarily strongly convex).

Note also that throughout the paper we assume that a stationary point of function $g$ exist. This in combination with Assumption~\ref{criticalminmax} guarantees that the Hamiltonian function $\cH(x)$ has at least on minimum point $x^*$. That is, there exist $x^*$ such that $\nabla \cH(x^*)=\bQ x^*+ q=0$.

Here matrix $\bQ$ is symmetric and $\bQ\succeq 0$. Thus, let $\bQ= \bL_{\bQ}^\top \bL_{\bQ}$ be the Cholesky decomposition of matrix $\bQ$. In addition, since we already assume that the Hamiltonian function $\cH(x)$ has at least on minimum point $x^*$ we have that  $q=-\bQ x^*= - \bL_{\bQ}^\top \bL_{\bQ} x^*$.

Using this note that:
$$\cH(x)= \phi(\bL_{\bQ} x)$$
where function $\phi(y)=\frac{1}{2}\|y\|^2-(\bL_{\bQ} x^*)^\top y+\ell$. In addition, note that function $\phi$ is $1$-strongly convex with $1$-Lipschitz continuous gradient.  

Thus, using Lemma~\ref{Necoara} we have that the the Hamiltonian function is a $L_{\cH}-$smooth, $\mu_{\cH}$--quasi-strongly convex function with constants $L_{\cH}=\|\bL_{\bQ}\|^2=\lambda_{\max}( \bL_{\bQ}^\top \bL_{\bQ})=\lambda_{\max}(\bQ)=\sigma_{\max}^2 (\bA)$ and $\mu_{\cH} = \sigma_{\min}^2(\bL_{\bQ})=\lambda_{\min}^+(\bQ)=\sigma_{\min}^2(\bA)$.

This completes the proof.
\end{proof}

\subsection{Proof of Proposition~\ref{SufficientlyBilinearGameProposition}}
\begin{proof}
Our proof follows closely the approach of \citet{abernethy2019last}. The main difference is that our game by its definition is stochastic and as a result quantities like $\bar{S}=\Exp_i[S_i]$ and $\bar{L}=\Exp_i[L_i]$ appear in the expression of $L_{\cH}$.

We divide the proof into two parts. In the first one we show that the Hamiltonian function is smooth and in the next one that satisfies the PL condition.

\paragraph{Smoothness.} Note that:
\begin{eqnarray}
\label{nalksxalx23}
\nabla \cH(x)&\overset{\eqref{nalksxalx2}}{=}&\frac{1}{n^2} \sum_{i=1}^n \sum_{j=1}^n \underbrace{\frac{1}{2}\left[ \bJ_i^\top \xi_j +   \bJ_j^\top \xi_i  \right]}_{\nabla \cH_{i,j}(x)}=\frac{1}{n^2} \left[  \frac{1}{2} \sum_{i=1}^n \sum_{j=1}^n \bJ_i^\top \xi_j +   \frac{1}{2} \sum_{j=1}^n \sum_{i=1}^n \bJ_j^\top \xi_i  \right]=\frac{1}{n^2} \left[ \sum_{i=1}^n \sum_{j=1}^n \bJ_i^\top \xi_j  \right]
\end{eqnarray}
Thus,
\begin{eqnarray}
\left\|\nabla \cH(x)-\nabla \cH(y)\right\|&\overset{\eqref{nalksxalx23}}{=}&\left\|\frac{1}{n^2} \sum_{i=1}^n \sum_{j=1}^n \bJ_i^\top(x) \xi_j(x) -\frac{1}{n^2} \sum_{i=1}^n \sum_{j=1}^n \bJ_i^\top(y) \xi_j(y)  \right\|\notag\\
&=&\left\|\frac{1}{n^2} \sum_{i=1}^n \sum_{j=1}^n \left[ \bJ_i^\top(x) \xi_j(x) -\bJ_i^\top(y) \xi_j(y)  \right] \right\|\notag\\
&\overset{Jensen}{\leq}& \frac{1}{n^2} \sum_{i=1}^n \sum_{j=1}^n \left\| \bJ_i^\top(x) \xi_j(x) -\bJ_i^\top(y) \xi_j(y) \right\|\notag\\
&\overset{(*)}{\leq}& \Exp_i \Exp_j \left\| \bJ_i^\top(x) \xi_j(x) -\bJ_i^\top(y) \xi_j(y) \right\|\notag\\
&=& \Exp_i \Exp_j \left\| \bJ_i^\top(x) \xi_j(x) +\bJ_i^\top(y) \xi_j(x)-\bJ_i^\top(y) \xi_j(x) -\bJ_i^\top(y) \xi_j(y) \right\|\notag\\
&=& \Exp_i \Exp_j \left\| \left[\bJ_i^\top(x)-\bJ_i^\top(y)\right] \xi_j(x) + \bJ_i^\top(y) \left[\xi_j(x)-\xi_j(y)\right] \right\|\notag\\
&\leq& \Exp_i \Exp_j \left\| \left[\bJ_i^\top(x)-\bJ_i^\top(y)\right] \xi_j(x)\right\| +\Exp_i \Exp_j \left\| \bJ_i^\top(y) \left[\xi_j(x)-\xi_j(y)\right] \right\|\notag\\
&=& \Exp_i \left\|\bJ_i^\top(x)-\bJ_i^\top(y)\right\| \Exp_j \left\|\xi_j(x)\right\| +\Exp_i \left\| \bJ_i^\top(y) \right\| \Exp_j \left\|\xi_j(x)-\xi_j(y)\right\|\notag\\
&\overset{Assumption~\ref{AssumptionOnGi}}{\leq} & \Exp_i \left[S_i \left\|x-y\right\|\right] \Exp_j \left\|\xi_j(x)\right\| +\Exp_i [L_i] \Exp_j \left[L_j \left\|x-y\right\|\right]\notag\\
&\overset{\Exp_j \left\|\xi_j(x)\right\|<C}{\leq} & C \bar{S}\left\|x-y\right\| +\bar{L}^2 \left\|x-y\right\|\notag\\
&= & \left(\bar{S} C+\bar{L}^2\right) \left\|x-y\right\|
\end{eqnarray}
where in $(*)$ we use that $i$ and $j$ are sampled from a uniform distribution.

\paragraph{PL Condition.}
To show that the Hamiltonian function satisfies the PL condition \eqref{Polyak} we use a linear algebra lemma from \cite{abernethy2019last}.

\begin{lemma}(Lemma H.2 in \cite{abernethy2019last})
\label{AbernethyLemma}\\
Let matrix $\bM = \begin{pmatrix} 
\bA   & \bC\\ 
-\bC^\top & -\bB\\ 
\end{pmatrix}$ where matrix $\bC$ is a square full rank matrix. Let $c=\left(\sigma^2_{\min}(C)+\lambda_{\min}(\bA^2)\right) \left(\lambda_{\min}(\bB^2)+\sigma^2_{\min}(\bC)\right)-\sigma_{\max}^2 (\bC) \left(\|\bA\|+\|B\|\right)^2$ and let assume that $c>0$. Here $\lambda_{\min}$ denotes the smaller eigenvalue and $\sigma_{\min}$ and $\sigma_{\max}$ the smallest and largest singular values respectively. Then if $\lambda$ is an eigenvalue of $\bM \bM^\top$ it holds that:
$$\lambda > \frac{\left(\sigma^2_{\min}(C)+\lambda_{\min}(\bA^2)\right)\left(\lambda_{\min}(\bB^2)+\sigma^2_{\min}(\bC)\right)-\sigma_{\max}^2 (\bC) \left(\|\bA\|+\|B\|\right)^2}{\left( 2\sigma^2_{\min} (\bC) +\lambda(\bA^2) +\lambda_{\min}(\bB^2) \right)^2}$$
\end{lemma}

In addition note that if there exist $\mu >0$ such that $\bJ(x) \bJ^\top(x) \succeq \mu \bI $ then the Hamiltonian function satisfies the PL condition with parameter $\mu$.

\begin{lemma}
\label{LemmaPLfunction}
Let $g(x)$ of min-max problem~\ref{MainStochasticProblem} be twice differentiable function. If there exist $\mu >0$ such that $\bJ(x) \bJ^\top(x) \succeq \mu \bI $ for all $x \in \R^d$ then the Hamiltonian function $\cH(x)$ \eqref{StochHamiltonianFunction} satisfies the PL condition \eqref{Polyak} with parameter $\mu$.
\end{lemma}
\begin{proof}
\begin{eqnarray}
\frac{1}{2}\|\nabla \cH(x)\|^2=\frac{1}{2}\| \bJ^\top(x) \xi(x)\|^2 &=& \frac{1}{2}  \xi(x)^\top  \bJ(x) \bJ^\top(x) \xi(x) \notag\\
&\overset{\bJ(x) \bJ^\top(x) \succeq \mu \bI }{\geq} &\frac{\mu}{2}  \xi(x)^\top \xi(x)\notag\\
&= &\mu \frac{1}{2}  \|\xi(x)\|^2 \notag\\
&=& \mu \left[\cH(x)\right]\notag\\
&\overset{\cH(x^*)=0}{=}& \mu \left[\cH(x)-\cH(x^*)\right]
\end{eqnarray}
\end{proof}

Combining the above two lemmas we can now show that for the sufficiently bilinear games that Hamiltonian function satisfies the PL condition with parameter $\mu_{\cH}=\frac{(\delta^2 +\rho^2)(\delta^2 +\beta^2) -4L^2 \Delta^2 }{2\delta^2+\rho^2+\beta^2}$.

Recall that $\bJ=\nabla \xi=\begin{pmatrix} 
\nabla^2_{x_1,x_1} g & \nabla^2_{x_1,x_2} g \\ 
-\nabla^2_{x_2,x_1} g & -\nabla^2_{x_2,x_2} g  \\ 
\end{pmatrix} \in \R^{d \times d}
$. Now let $\bC(x)=\nabla^2_{x_1,x_2} g(x)$ and note that this is a square full rank matrix. In particular, by the assumption of the sufficiently bilinear games we have that the cross derivative $\nabla^2_{x_1,x_2} g$ is full rank matrix with $0 < \delta \leq \sigma_i \left(\nabla^2_{x_1,x_2} g\right)  \leq \Delta$ for all $x \in \R^d$ and for all singular values $\sigma_i $. In addition since we assume that the sufficiently bilinear condition \eqref{SufficientBilinear} holds we can apply Lemma~\ref{AbernethyLemma} with $\bM=\bJ$. Since function $g$ is smooth in $x_1$ and $x_2$ and using the bounds on the singular values of matrix $\bC(x)$ we have that:
$$ \bJ \bJ^\top \succeq \left[\frac{(\delta^2 +\rho^2)(\delta^2 +\beta^2) -4L^2 \Delta^2 }{2\delta^2+\rho^2+\beta^2}\right] \bI,$$
where $\rho^2 = \min_{x_1,x_2} \lambda_{\min} \left[ \nabla^2_{x_1,x_1} g(x_1,x_2)\right]^2$ and $\beta^2 = \min_{x_1,x_2} \lambda_{\min} \left[ \nabla^2_{x_2,x_2} g(x_1,x_2)\right]^2$.
Using Lemma~\ref{LemmaPLfunction} it is clear that the Hamiltonian function of the sufficiently bilinear games satisfies the PL condition with $\mu_{\cH}=\frac{(\delta^2 +\rho^2)(\delta^2 +\beta^2) -4L^2 \Delta^2 }{2\delta^2+\rho^2+\beta^2}$. 

This completes the proof.
\end{proof}

\section{Proofs of Main Theorems}
\label{sec:analysisAppendix}
In this section we present the proofs of the convergence analysis Theorems  presented in Section~\ref{sec:analysis} for the convergence of SHGD (constant and decreasing step-size) and L-SVRHG (and its variant for PL functions) for solving the bilinear games and sufficiently bilinear games.

\subsection{Derivation of Convergence Results}
The Theorems of the papers can be obtained by combining existing and new optimization convergence results with the two main proposition proved in the previous section (Propositions~\ref{BilinearGameProposition} and~\ref{SufficientlyBilinearGameProposition}). 

In particular we use the following combination of results:
\begin{itemize}
\item \textbf{For the Bilinear Games:}
\begin{itemize}
\item Convergence of SHGD with constant step-size (Theorem~\ref{theo:strcnvlin}): Combination of constant step-size SGD theorem from \citet{gower2019sgd} and Proposition~\ref{BilinearGameProposition}.
\item Convergence of SHGD with decreasing step-size (Theorem~\ref{theo:decreasingstep}): Combination of decreasing step-size SGD theorem from \citet{gower2019sgd} and Proposition~\ref{BilinearGameProposition}.
\item Convergence of L-SVRHG (Theorem~\ref{TheoremLSVRHGBilinear}): Combination of the convergence of L-SVRG from \citet{kovalev2019don, gorbunov2020unified} and Proposition~\ref{BilinearGameProposition}.
\end{itemize}
\item \textbf{For the Sufficiently Bilinear Games:}
\begin{itemize}
\item Convergence of SHGD with constant step-size (Theorem~\ref{SGDforPolyak}): Combination of constant step-size SGD theorem from \citet{gower2020sgd} and Proposition~\ref{SufficientlyBilinearGameProposition}.
\item Convergence of SHGD with decreasing step-size (Theorem~\ref{theo:decreasingstepPL}): Combination of decreasing step-size SGD theorem from \citet{gower2020sgd} and Proposition~\ref{SufficientlyBilinearGameProposition}.
\item Convergence of L-SVRHG with Restarts (Theorem~\ref{TheoremLSVRHGforPL}):  Combination of Theorem~\ref{NewTheoremPL} describing the convergence of Algorithm~\ref{LSVRG_Restart} in the optimization setting (extension of the convergence results from \citet{qian2019svrg}) and Proposition~\ref{SufficientlyBilinearGameProposition}.  
\end{itemize}
\end{itemize} 

In the rest of this section we present the Theorems of the convergence of SGD (Algorithm~\ref{SGD_Algorithm}) and L-SVRG (Algorithm~\ref{LSVRG_Algorithm}) for solving the finite sum problem \eqref{FiniteSumOPT} as presented in the above papers with some brief comments on their convergence. As we explain above combining these results with the Propositions~\ref{BilinearGameProposition} and~\ref{SufficientlyBilinearGameProposition} yield the Theorems presented in Section~\ref{sec:analysis}.

\subsection{Convergence of Stochastic Optimization Methods for $\mu$--Quasi-strongly Convex Functions}
In this subsection we present the main convergence results as presented in \citet{gower2019sgd} for SGD and in \citet{kovalev2019don, gorbunov2020unified} for L-SVRG. The main assumption of these Theorems is the function $f$ of problem \eqref{FiniteSumOPT} to be $\mu$--quasi-strongly convex function and that the expected smoothness is satisfied. Note that no assumption on convexity of $f_i$ is made.

\subsubsection{Convergence of SHGD}
Two Theorems have been presented in \citet{gower2019sgd} for the convergence of SGD one for constant step-size and one with a decreasing step-size. In particular the second theorem provide insightful  stepsize-switching rules  which describe when one should switch from a constant to a decreasing stepsize regime. As expected for the choice of constant step-size, SGD converges with linear rate to a neighborhood of the solution while for the decreasing step-size converges to the exact optimum but with a slower sublinear rate.

For the case of our stochastic Hamiltonian methods this is exactly the behavior of the SHGD where for the constant step-size the method convergence to a neighborhood of the min-max solution while for the case of decreasing step-size the method converges with a slower rate to the min-max solution of problem \eqref{MainStochasticProblem}.

\begin{theorem}[Constant Stepsize]
\label{theo:strcnvlinOPT}
Assume $f$ is $\mu$-quasi-strongly convex and that $(f,\cD)\sim ES(\cL)$.
Choose   $\gamma^k=\gamma \in (0,  \frac{1}{2\cL}]$ for all $k$. Then iterates of Stochastic Gradient Descent (SGD) given in Algorithm~\ref{SGD_Algorithm} satisfy: 
\begin{equation}\label{eq:convsgdOPT}
\mathbb{E} \| x^k - x^* \|^2 \leq \left( 1 - \gamma \mu \right)^k \| x^0 - x^* \|^2 + \tfrac{2 \gamma \sigma^2}{\mu}. 
\end{equation}
\end{theorem}

\begin{theorem}[Decreasing stepsizes/switching strategy]
\label{theo:decreasingstepOPT}
Assume $f$ is $\mu$-quasi-strongly convex and that $(f,\cD)\sim ES(\cL)$. Let   $\mathcal{K} \eqdef \left.\cL\right/\mu$ and 
\begin{equation}\label{eq:gammakdefOPT}
\gamma^k= 
\begin{cases}
\displaystyle \tfrac{1}{2\cL} & \mbox{for}\quad k \leq 4\lceil\mathcal{K} \rceil \\[0.3cm]
\displaystyle \tfrac{2k+1}{(k+1)^2 \mu} &  \mbox{for}\quad k > 4\lceil\mathcal{K} \rceil.
\end{cases}
\end{equation}
If $k \geq 4 \lceil\mathcal{K} \rceil$, then Stochastic Gradient Descent (SGD) given in Algorithm~\ref{SGD_Algorithm} satisfy:
\begin{equation}\label{eq:rateofdecreasingOPT}
\mathbb{E}\| x^{k} - x^*\|^2 \le   \tfrac{\sigma^2 }{\mu^2 }\tfrac{8 }{k} + \tfrac{16 \lceil\mathcal{K} \rceil^2}{e^2 k^2 }  \|x^0 - x^*\|^2.\end{equation}
\end{theorem}
\subsubsection{Convergence of L-SVRG}
As we explained in the main paper L-SVRG is a variance reduced method which means that allow us to obtain convergence to the exact solution of the problem. The methods is analyzed in \citet{kovalev2019don} for the case of strongly convex functions and extended to the class of $\mu$-quasi strongly convex in \citet{gorbunov2020unified}. Following the theorem proposed in \citet{gorbunov2020unified} it can be shown that the method converges to the solution $x^*$ with linear rate as follows:
\begin{theorem}
Assume $f$ is $\mu$-quasi-strongly convex. Let step-size $\gamma= 1/6L$ and let $p \in (0,1]$ and $\cD$ be the uniform distribution. Then L-SVRG$_{I}$ given in Algorithm~\ref{LSVRG_Algorithm} convergences to the optimum and satisfies:
$$\Exp[\Phi^k]\leq \max \{1-\mu/6L , 1-p/2\}^k \Phi^0$$
where $\Phi^k=\|x^k-x^*\|^2+\frac{4\gamma^2}{p n}\sum_{i=1}^{n}\|\nabla f_i(w^k)-\nabla f_i(x^*)\|^2$.
\end{theorem}
Note that in statement of Theorem~\ref{TheoremLSVRHGBilinear} we replace $n$ in the above expression with $n^2$ because the Hamiltonian function has finite-sum structure with $n^2$ components. As we explain before for obtaining Theorem~\ref{TheoremLSVRHGBilinear} of the main paper one can simply combine the above Theorem with Proposition~\ref{BilinearGameProposition}. 

We highlight that in~\cite{qian2019svrg} a convergence Theorem of L-SVRG for smooth strongly convex functions was presented under the arbitrary sampling paradigm (well defined distribution $\cD$) . This result can be trivially extended to capture the class of smooth quasi-strongly convex functions and as such it can be also used in combination with \ref{BilinearGameProposition}. In this case the step-size will become $\gamma= 1/6\cL$ where $\cL$ is the expected smoothness parameter. Using this, one can  guarantee linear convergence of L-SVRG, and as a result of L-SVRHG, with more general distribution $\cD$ (beyond uniform sampling). For other well defined choices of distribution $\cD$ we refer the interested reader to~\cite{qian2019svrg}.

\subsection{Convergence of Stochastic Optimization Methods for Functions Satisfying the PL Condition}
\label{app:sufficiently-bilinear-games}
As we have already mentioned the Theorems of the convergence of the stochastic Hamiltonian methods for solving the sufficiently bilinear games can be obtain by combining Proposition~\ref{SufficientlyBilinearGameProposition} with existing results on the analysis of SGD and L-SVRG for functions satisfying the PL condition. 

In particular, in this subsection we present the main convergence Theorems as presented in \citet{gower2020sgd} for the analysis of SGD for functions satisfying the PL condition and we explain how we can extend the results of \citet{qian2019svrg} in order to provide an analysis of L-SVRG with restart.

The main assumption of these Theorems is that function $f$ of problem \eqref{FiniteSumOPT} satisfies the PL condition and that the expected residual is satisfied. Note that again no assumption on convexity of $f_i$ is made. 

An important remark that we need to highlight is that all convergence result are presented in terms of function suboptimality $\Exp[f(x^{k})-f(x^*)]$. When these results are used for the Hamiltonian method that we know that $\cH(x^*)=0$ they can be written as $\Exp[\cH(x^{k})]$. This is exactly the quantity for which we show convergence in Theorems~\ref{SGDforPolyak}, \ref{theo:decreasingstepPL} and  \ref{TheoremLSVRHGforPL}.

\subsubsection{Convergence of SHGD}
In \citet{gower2020sgd} several convergence theorems describing the performance of SGD were presented for two large classes of structured nonconvex functions: (i) the Quasar (Strongly) Convex functions and (ii) the functions satisfying the Polyak-Lojasiewicz (PL) condition.  The proposed analysis of~\citet{gower2020sgd} relied on the Expected Residual assumption~\ref{eq:expresidual}. The authors proved the weakness of this assumption compared to previous one used for the analysis of SGD and highlight the benefits of using it when the function satisfying the PL condition.

 In particular, one of the main contributions of this work is a novel analysis of minibatch SGD for PL functions which recovers the best known dependence on the condition number for Gradient Descent (GD) while also matching the current state-of-the-art rate derived in for SGD. Recall that this is what we used to obtain the convergence of deterministic Hamiltonian method (equivalent to GD) in Corollary~\ref{ncaoskla}.

In \citet{gower2020sgd} two theorems have been proposed for PL functions for the convergence of SGD, one for constant step-size and one with a decreasing step-size. In particular the second theorem provide insightful  stepsize-switching rules  which describe when one should switch from a constant to a decreasing stepsize regime. For the case of constant step-size, SGD converges with linear rate to a neighborhood of the solution while for the decreasing step-size converges to the exact optimum but with a slower sublinear rate.

\begin{theorem}
\label{SGDforPolyakOptimization}
Let $f$ be $L$-smooth. Assume expected residual and that $f(x)$ satisfies the PL condition~\eqref{PLcondition}. Set $\sigma^2 = \mathbb{E}(\| \nabla f_i(x^*)\|^2 )$, where $x^* = \arg \min_x f(x)$.  Let $\gamma_k = \gamma\leq \frac{\mu}{L (\mu +2\rho)},$ for all $k$. Then the iterates of SGD satisfy: 
\begin{equation}
\Exp[f(x^{k})-f^*] \leq \left(1- \gamma \mu \right)^k [f(x^0)-f^*] + \frac{L \gamma \sigma^2} {\mu}.
\end{equation}
\end{theorem}

\begin{theorem}[Decreasing step sizes/switching strategy]
\label{theo:decreasingstepPLOPT}
Let $f$ be an $L$-smooth. Assume expected residual and that $f(x)$ satisfies the PL condition~\eqref{PLcondition}. Let $k^* \eqdef 2\frac{L}{\mu} \left(1+2\tfrac{\rho}{\mu}\right)$ and 
\begin{equation}
\gamma^k= 
\begin{cases}
\displaystyle  \frac{\mu}{L (\mu +2\rho)}, & \mbox{for}\quad k \leq \lceil k^* \rceil\\[0.3cm]
\displaystyle \frac{2k+1}{(k+1)^2 \mu} &  \mbox{for}\quad k >  \lceil k^* \rceil
\end{cases}
\end{equation}
If $k \geq  \lceil k^*  \rceil$, then SGD given in Algorithm~\ref{SGD_Algorithm} satisfies:
\begin{equation}\label{eq:decreasingstepPLExpResidual}
\Exp[f(x^{k})-f^*] \le  \frac{4 L \sigma^2 }{\mu^2 }\tfrac{1}{k} + \tfrac{( k^*)^2}{k^2 e^2}  [f(x^0)-f^*]  .
\end{equation}
\end{theorem}

\subsubsection{Convergence of L-SVRG (with Restart)}
In \citet{qian2019svrg} an analysis of L-SVRG (Algorithm~\ref{LSVRG_Algorithm}) was provided for non-convex and smooth functions. In particular Lemma~\ref{XunTheorem} has been proved under the Expected residual Assumption~\ref{onklm}.

\begin{assumption}
\label{onklm}
There is a constant $\rho_{nc}>0$ such that
\begin{equation}
\Exp \left[\left\| \nabla f_i(x)-\nabla f_i(y) -\left[\nabla f(x) -\nabla f(y)\right]\right\|^2 \right] \leq \rho_{nc} \|x-y||^2
\end{equation}
\end{assumption}
Assumption~\ref{onklm} is similar to the expected residual condition presented in the main paper. For the case of $\tau$-minibatch sampling, in~\citet{qian2019svrg} it was shown that parameter $\rho_{nc}$ of the assumption can be upper bounded by $\rho_{nc}\leq \tfrac{n^2-\tau}{(n^2-1)\tau}\frac{1}{n} \sum L_i^2$ where $L_i$ is the smoothness parameter of function $f_i$.

Under the Expected residual Assumption~\ref{onklm} the following lemma was proven in~\citet{qian2019svrg}.

\begin{lemma}[Theorem 5.1 in \citet{qian2019svrg} ]
\label{XunTheorem}
Let $f$ be nonconvex and smooth function. Let Assumption~\ref{onklm} be satisfied and let $p \in (0,1]$. Consider the Lyapunov function $\Psi^k=f(x^k)+\alpha \|x^k-w^k\|^2$ where $\alpha =3\gamma^2 L \rho_{nc} /p$. If stepsize $\gamma$ satisfies:
\begin{equation}
\label{cnaosnkl}
\gamma \leq \min \left\{  \frac{1}{4L}, \frac{p^{2/3}}{36^{1/3}(L \rho_{nc} )^{1/3}},\frac{\sqrt{p}}{\sqrt{6\rho_{nc}}} \right\}
\end{equation}
then the update of L-SVRG (Algorithm~\ref{LSVRG_Algorithm}) satisfies $$\Exp_i[\Psi^{k+1}]\leq \Psi^{k} -\frac{\gamma}{4}\|\nabla f(x^k)\|^2.$$
\end{lemma}

Having the result of Lemma~\ref{XunTheorem} let us now present the main Theorem describing the convergence of L-SVRG with restart presented in Algorithm~\ref{LSVRG_Restart}.
Let us, run L-SVRG with step-size $\gamma$ that satisfies \eqref{cnaosnkl} and select the output $x^u$ of the method to be its Option II. That is $x^u$ is chosen uniformly at random from $\{x^i\}^K_{i=0}$. 
In this case we name the method L-SVRG$_{II}( x^0=w^0, K, \gamma, p \in (0,1] )$.

\begin{theorem}[Convergence of Algorithm~\ref{LSVRG_Restart}]
\label{NewTheoremPL}
Let $f$ be $L-$smooth function that satisfies the PL condition~\eqref{PLcondition} with parameter $\mu$. Let Assumption~\ref{onklm} be satisfied and let $p \in (0,1]$. If stepsize $\gamma$ satisfies:
\begin{equation*}
\gamma \leq \min \left\{  \frac{1}{4L}, \frac{p^{2/3}}{36^{1/3}(L \rho_{nc} )^{1/3}},\frac{\sqrt{p}}{\sqrt{6\rho_{nc}}} \right\}
\end{equation*}
and $K=\frac{4}{\mu \gamma}$ then the update of Algorithm~\ref{LSVRG_Restart} satisfies 
\begin{equation}
\label{bcoakldal}
\Exp[f(x^{t})-f(x^*)] \le \left(\frac{1}{2}\right)^t  [f(x^{0})-f(x^*)],
\end{equation}
and \begin{equation}
\label{bcoakldal2}
\Exp\|\nabla f(x^t)\|^2 \le \left(\frac{1}{2}\right)^t   \|\nabla f(x^{0})\|^2.
\end{equation}
\end{theorem} 

\begin{proof}
Using Lemma~\ref{XunTheorem} we obtain:
$$\Exp_i[\Psi^{k+1}]\leq \Psi^{k} -\frac{\gamma}{4}\|\nabla f(x^k)\|^2.$$
By taking expectation again and by rearranging:
$$\Exp\|\nabla f(x^k)\|^2 \leq \frac{4}{\gamma} \left[\Exp\left[\Psi^{k}\right] -\Exp[\Psi^{k+1}]\right]$$
By letting $x^u$ to be chosen uniformly at random from $\{x^i\}^K_{i=0}$ we obtain:
\begin{eqnarray}
\label{noaksdaod}
\Exp\|\nabla f(x^u)\|^2 &\leq& \frac{1}{K} \sum_{i=0}^{K-1} \Exp\|\nabla f(x^i)\|^2 \notag\\
&\leq& \frac{1}{K} \frac{4}{\gamma} \sum_{i=0}^{K-1} \left[\Exp\left[\Psi^{k}\right] -\Exp[\Psi^{k}]\right]\notag\\
&=& \frac{1}{K} \frac{4}{\gamma}  \left(\Psi^{0} -\Exp[\Psi^{K}]\right)\notag\\
&=& \frac{1}{K} \frac{4}{\gamma}  \left(f(x^0) -\Exp[f(x^k)]- \alpha \Exp[\|x^k-w^k\|^2] \right)\notag\\
&\leq& \frac{1}{K} \frac{4}{\gamma}  \left[ f(x^0) -f(x^*) \right]
\end{eqnarray}

\paragraph{Convergence on function values.}
The above derivation, \eqref{noaksdaod}, shows that the iterates of Algorithm~\ref{LSVRG_Restart} satisfy:
$$
\Exp\|\nabla f(x^t)\|^2 \leq \frac{4}{\gamma K} \Exp\left[ f(x^{t-1}) -f(x^*) \right]$$
Substitute the specified value of $K=\frac{4}{\gamma \mu}$ in the above inequality, we have
$$\Exp\|\nabla f(x^t)\|^2  \leq  \mu \Exp\left[ f(x^{t-1}) -f(x^*) \right]$$
and since the function satisfies the PL condition we have $\frac{1}{2}\|\nabla f(x)\|^2 \geq \mu \left[f(x)-f(x^*)\right]$ which means that: $$ 2 \mu \Exp\left[f(x^t)-f(x^*)\right] \leq \Exp\|\nabla f(x^t)\|^2  \leq  \mu \Exp\left[ f(x^{t-1}) -f(x^*) \right]$$
Thus, $$ \Exp\left[f(x^t)-f(x^*)\right]  \leq \left(\frac{1}{2}\right) \Exp\left[ f(x^{t-1}) -f(x^*) \right]$$
by unrolling the recurrence we obtain \eqref{bcoakldal}.

\paragraph{Convergence on norm of the gradient.}
Similar to the previous case, using \eqref{noaksdaod}, the iterates of Algorithm~\ref{LSVRG_Restart} satisfy:
\begin{eqnarray}
\Exp\|\nabla f(x^t)\|^2 &\leq& \frac{4}{\gamma K} \left[ f(x^{t-1}) -f(x^*) \right]\notag\\
&\overset{\eqref{PLcondition}} {\leq} & \frac{4}{\gamma K} \frac{1}{2\mu}\|\nabla f(x^{t-1})\|^2 \notag\\
&=& \frac{2}{\gamma \mu K} \|\nabla f(x^{t-1})\|^2 
\end{eqnarray}
Using the specified value $K=\frac{4}{\gamma \mu}$ in the above inequality, we have:
\begin{eqnarray}
\Exp\|\nabla f(x^t)\|^2 &\leq& \left(\frac{1}{2}\right) \|\nabla f(x^{t-1})\|^2
\end{eqnarray}
and by unrolling the recurrence we obtain \eqref{bcoakldal2}.
\end{proof}

\section{Experimental Details}
\label{app:experiments-details}

In the experimental section we compare several different algorithms, we provide a short explanation of the different algorithms here:
\begin{itemize}
    \item \textbf{SHGD} with constant and decreasing step-size: This is the Alg.~\ref{SHGD_Algorithm} proposed in the paper.
    \item \textbf{Biased SHGD}: This is a biased version of Alg.~\ref{SHGD_Algorithm} that was proposed by \citet{mescheder2017numerics}, where $\nabla \cH_{i,j}(x) = \frac{1}{2}\nabla \langle  \xi_i(x),  \xi_j(x)\rangle $ is replaced by $\nabla \hat{\cH}_{i,j}(x) = \frac{1}{2}\nabla \|\xi_i(x) + \xi_j(x)\|^2$, note that this a biased estimator of $\nabla \cH(x)$.  
    \item \textbf{L-SVRHG} with or without restart: This is the Alg.\ref{LSVRHG_Algorithm} proposed in the paper, with Option II for the restart and Option I for the version without restart. Restart is not used unless specified.
    \item \textbf{CO}: This is the Consensus Optimization algorithm proposed in \citet{mescheder2017numerics}. We provide more details in App.~\ref{app:CO}.
    \item \textbf{SGDA}: This is the stochastic version of Simultaneous Gradient Descent/Ascent algorithm, which uses the following update $x^{k+1}=x^k - \eta_k \xi_i(x_k)$.
    \item \textbf{SVRE} with restart: This is the Alg. 3 described in \citet{chavdarova2019reducing}.
\end{itemize}

In the following sections we provide the details for the different hyper-parameters used in our different experiments. 

\subsection{Bilinear Game}
\label{app:bilinear-games}

We first provide the details about the bilinear experiments presented in section~\ref{exp:bilinear-games}:
\begin{equation}
   \min_{x_1 \in\R^{d}} \max_{x_2 \in\R^{d}} \frac{1}{n} \sum_{i=1}^n  x_1^\top \bA_i x_2 + b_i^\top x_1 + c_i^\top x_2
\end{equation}
where:
$$n=d=100$$
\begin{equation*}
    \bA_i \in R^{d\times d}, \quad    [\bA_i]_{kl} = \left\{\begin{array}{ll}
    {1} & {\text{if } i = k = l }\\
    {0} & {\text{otherwise}} 
\end{array}\right.
\end{equation*}
$$ b_i, c_i \in R^d, \quad [b_i]_k, [c_i]_k \sim \mathcal{N}(0, 1/d)$$

The hyper-parameters used for the different algorithms are described in Table~\ref{tab:bilinear-hyperparameters}:

\bgroup
\def\arraystretch{1.5}
\begin{table}[H]
\vspace{-2mm}
\caption{Hyper-parameters used for the different algorithms in the Bilinear Experiments (section~\ref{exp:bilinear-games}).}
\label{tab:bilinear-hyperparameters}
\vspace{-1mm}
\begin{center}
\begin{small}
\begin{sc}
\begin{tabular}{llcc}
\toprule
Algorithms & Step-size $\gamma^k$ & Probability $p$ & Restart\\
\midrule
SHGD with constant step-size   & $0.5$ & n/a & n/a \\
SHGD with decreasing step-size & $\begin{cases}
\displaystyle 0.5 & \mbox{for}\quad k \leq 10,000 \\
\displaystyle \tfrac{2k+1}{(k+1)^2 \frac{1}{2500}} &  \mbox{for}\quad k > 10,000.
\end{cases}$ & n/a & n/a\\[0.4cm]
Biased SHGD & $0.5$ & n/a & n/a\\
L-SVRHG    & $10$ & $\frac{1}{n}=0.01$ & n/a\\
SVRE    & $0.3$ & $\frac{1}{n}=0.01$ & restart with probability $0.1$ \\
\bottomrule
\end{tabular}
\end{sc}
\end{small}
\end{center}
\end{table}
\egroup

The optimal constant step-size suggested by the theory for \textbf{SHGD} is $\gamma = \tfrac{1}{2\cL}$. In this experiment we have that $\cL=1$, thus the optimal step-size is $0.5$, this is also what we observed in practice. However we observe that while the theory recommends to decrease the step-size after $4 \lceil\mathcal{K} \rceil = 40,000$ we observe in this experiment that it actually converges faster if we decrease the step-size a bit earlier after only $10,000$ iterations.

\subsection{Sufficiently-Bilinear Games}
\label{app:nonlinear-games}

In this section we provide more details about the sufficiently-bilinear experiments of section~\ref{exp:nonlinear-games}:
\begin{equation*}
 \min_{x_1 \in\R^{d}} \max_{x_2 \in\R^{d}}  \frac{1}{n} \sum_{i=1}^n  F(x_1) + \delta x_1^\top \bA_i x_2 + b_i^\top x_1 + c_i^\top x_2 - F(x_2)   
\end{equation*}

where:
$$n=d=100 \quad \text{and} \quad \delta=7$$
\begin{equation*}
    \bA_i \in R^{d\times d}, \quad    [\bA_i]_{kl} = \left\{\begin{array}{ll}
    {1} & {\text{if } i = k = l }\\
    {0} & {\text{otherwise}} 
\end{array}\right.
\end{equation*}
$$ b_i, c_i \in R^d, \quad [b_i]_k, [c_i]_k \sim \mathcal{N}(0, 1/d)$$

\begin{equation}
F(x) = \frac{1}{d} \sum_{k=1}^d f(x_k), a\quad
f(x)=\left\{\begin{array}{ll}
{-3\left(x+\frac{\pi}{2}\right)} & {\text { for } x \leq-\frac{\pi}{2}} \\
{-3 \cos x} & {\text { for }-\frac{\pi}{2}<x \leq \frac{\pi}{2}} \\
{-\cos x+2 x-\pi} & {\text { for } x>\frac{\pi}{2}}
\end{array}\right.
\end{equation}

Note that this game satisfies the sufficiently-bilinear condition as long as $\delta > 2L$, where $L$ is the smoothness of $F(x)$, in our case $L=3$. Thus we choose $\delta=7$ in order for the sufficiently-bilinear condition to be satisfied.

The hyper-parameters used for the different algorithms are described in Table~\ref{tab:sufficiently-bilinear-hyperparameters}:

\bgroup
\def\arraystretch{1.5}
\begin{table}[H]
\vspace{-2mm}
\caption{Hyper-parameters used for the different algorithms in the sufficiently-bilinear experiments (section~\ref{exp:nonlinear-games}).}
\label{tab:sufficiently-bilinear-hyperparameters}
\vspace{-1mm}
\begin{center}
\begin{small}
\begin{sc}
\begin{tabular}{llcc}
\toprule
Algorithms & Step-size $\gamma^k$ & Probability $p$ & Restart\\
\midrule
SHGD with constant step-size   & $0.02$ & n/a & n/a \\
SHGD with decreasing step-size & $\begin{cases}
\displaystyle 0.02 & \mbox{for}\quad k \leq 10,000 \\
\displaystyle \tfrac{2k+1}{(k+1)^2 \frac{1}{2500}} &  \mbox{for}\quad k > 10,000.
\end{cases}$ & n/a & n/a\\[0.4cm]
Biased SHGD & $0.01$ & n/a & n/a\\
L-SVRHG & $0.1$ & $\frac{1}{n}=0.01$ & n/a\\
L-SVRHG with restart   & $0.1$ & $\frac{1}{n}=0.01$ & restart every $1,000$ iterations\\
SVRE    & $0.05$ & $0.1$ & restart with probability $0.1$ \\
\bottomrule
\end{tabular}
\end{sc}
\end{small}
\end{center}
\end{table}
\egroup

\subsection{GANs}
\label{app:gans}

In this section we present the details for the GANs experiments. We first present the different problem we try to solve.

satGAN solve the following problem:
$$
\min_{\mu,\sigma} \max_{\phi_0,\phi_1,\phi_2} \frac{1}{n}\sum_{i=1}^{n}\log(\text{sigmoid}(\phi_0 + \phi_1 y_i + \phi_2 y_i^2)) + \log(1-\text{sigmoid}(\phi_0 + \phi_1 (\mu + \sigma z_i) + \phi_2 (\mu + \sigma z_i)^2))
$$

nsGAN solve the following problem:
\begin{align*}
\max_{\phi_0,\phi_1,\phi_2} \frac{1}{n}\sum_{i=1}^{n} &\log(\text{sigmoid}(\phi_0 + \phi_1 y_i + \phi_2 y_i^2)) + \log(1-\text{sigmoid}(\phi_0 + \phi_1 (\mu + \sigma z_i) + \phi_2 (\mu + \sigma z_i)^2)) \\
\max_{\mu,\sigma} \frac{1}{n}\sum_{i=1}^{n} &\log(\text{sigmoid}(\phi_0 + \phi_1 (\mu + \sigma z_i) + \phi_2 (\mu + \sigma z_i)^2)) + \log(1-\text{sigmoid}(\phi_0 + \phi_1 y_i + \phi_2 y_i^2)) 
\end{align*}

WGAN solve the following problem:
$$
\min_{\mu,\sigma} \max_{\phi_1,\phi_2} \frac{1}{n}\sum_{i=1}^{n} (\phi_1 y_i + \phi_2 y_i^2) - (\phi_1 (\mu + \sigma z_i) + \phi_2 (\mu + \sigma z_i)^2) 
$$

All Discriminator and Generator parameters are initialized randomly with $U(-1,1)$ prior. The data is set as $y_i \sim N(0,1)$, $z_i \sim N(0,1)$. We run all experiments 10 times (with seed 1, 2, \ldots, 10).

The hyper-parameters used for the different algorithms are described in Table~\ref{tab:gan-hyperparameters}:
\vspace{-2mm}

\bgroup
\def\arraystretch{1.5}
\begin{table}[H]
\caption{Hyper-parameters used for the different algorithms in the GAN Experiments (section~\ref{exp:GANs}).}
\label{tab:gan-hyperparameters}
%\vspace{-1mm}
\begin{center}
\begin{small}
\begin{sc}
\begin{tabular}{llcccc}
\toprule
Algorithms & Step-size $\gamma^k$ & Probability $p$ & Sample size & Mini-batch size \\
\midrule
CO   & .02 & n/a & 10K & 100 \\
SGDA & .02 & n/a& 10K & 100 \\
SHGD & .02 & n/a & 10K & 100\\
L-SVRHG & .02 & $\frac{1}{n}=0.01$ & 10K & 100\\
\bottomrule
\end{tabular}
\end{sc}
\end{small}
\end{center}
\end{table}
\egroup

\subsection{Implementation details for L-SVRHG}
\label{app:l-svrhg}
L-SVRHG requires the computation of the gradient of the full Hamiltonian with probability $p$.
As a reminder the Hamiltonian can be written as the sum of $n^2$ terms (see eq~\ref{StochHamiltonianFunction}), a naive implementations would thus requires $n^2$ operation to compute each of the $\nabla \cH_{i,j}(x)$ to get the full gradient. 
However a more efficient alternative is to notice that the Hamiltonian can be written as $\cH(x)= \frac{1}{2} \|\xi(x)\|^2$,  by first computing $\xi(x)$ and then using back-propagation to compute the gradient of  $\cH(x)$, we can reduce the cost of computing the full gradient to $2n$ instead of $n^2$.

\subsection{Details for Consensus Optimization (CO)}
\label{app:CO}

Consensus optimization can be formulated as solving the following problem using SGDA:
\begin{equation}
\label{ConsensusOptimization}
\min_{x_1 \in\R^{d_1}} \max_{x_2 \in\R^{d_2}} g(x_1, x_2) + \lambda \cH((x1,x2))
\end{equation}

Using SGDA to solve this problem is equivalent to do the following update:
\begin{equation}
\label{eq:co-update}
x^{k+1}=x^k - \eta_k (\xi(x_k) + \lambda \nabla \cH_{i,j}(x^k))
\end{equation}

As per \citet{mescheder2017numerics}, we used $\lambda=10$ in all the experiments and a biased estimator of the Hamiltonian $\nabla \hat{\cH}_{i,j}(x) = \frac{1}{2}\nabla \|\xi_i(x) + \xi_j(x)\|^2$. Note that we also tried to use the unbiased estimator proposed in section~\ref{unbiased-estimator} but found no significant difference in our results, and thus only included the results for the original algorithm proposed by \citet{mescheder2017numerics} that uses the biased estimator.

\subsection{Cost per iteration}
\label{app:cost-per-iteration}
In the experimental section, we compare the different algorithms as a function of the number of gradient computations. 
In Table~\ref{tab:cost-per-iteration}, we give the number of gradient computations per iteration for all the different algorithms compared in the paper.

We also give a brief explanation on the cost per iteration of each methods:
\begin{itemize}
    \item \textbf{SHGD}: We can write $\cH_{i,j}(x)= \frac{1}{2}<\xi_i(x), \xi_j(x)>$, thus at every iteration we need to compute two gradients $\xi_i(x)$ and $\xi_j(x)$, which leads to a cost of 2 per iteration.
    \item \textbf{Biased SHGD}: The biased estimate is based on $\cH_i(x)=\frac{1}{2}\|\xi_i(x) + \xi_j(x)\|^2$, which requires the computation of two gradients $\xi_i(x)$ and thus also has a cost of 2 per iteration.
    \item \textbf{L-SVRHG}: At every iteration we need to compute two Hamiltonian updates which cost $2$ each, and with probability $p$ we need to compute the full Hamiltonian which cost $2n$ (see App.~\ref{app:l-svrhg}), which leads to a cost of $4 + p \cdot 2n$
    \item \textbf{SVRE}: At each iteration \textbf{SVRE} need to do an extrapolation step and an update step, both the extrapolation step and the update step requires to evaluate two gradients, and with probability $p$ we need to compute the full gradient which cost $n$,  which leads to a total cost of $4 + p \cdot n$.
\end{itemize}
\begin{table}[H]
\caption{Number of the gradient computations per iteration for the different algorithms compared in the paper.}
\label{tab:cost-per-iteration}
\vspace{3mm}
\begin{center}
\begin{small}
\begin{sc}
\begin{tabular}{lcl}
\toprule
Algorithm & Cost per iteration\\
\midrule
SGDA & 1\\
SHGD    & $2$\\
Biased SHGD & $2$\\
L-SVRHG    & $4 + p \cdot 2n$\\
SVRE    & $4 + p \cdot n$\\
\bottomrule
\end{tabular}
\end{sc}
\end{small}
\end{center}
\end{table}

\section{Additional Experiments}
\label{AddExpAppendix}
In this section we provide additional experiments that we couldn't include in the main paper. Those experiments provide further observations on the behavior of our proposed methods in different settings.

\subsection{Bilinear and Sufficiently-Bilinear Games}

\subsubsection{Symmetric positive definite matrix} 

In all the experiments presented in the paper the matrix $\bA_i$ have a particular structure, they are very sparse, and the matrix $\bA = 1/n \sum_{i=1}^n \bA_i$ is the identity. We thus propose here to compare the methods on the bilinear game \eqref{bilinearGame1} and the sufficiently-bilinear game from Section~\ref{exp:nonlinear-games} but with different matrices $\bA_i$. We choose $\bA_i$ to be random symmetric positive definite matrices. For the sufficiently bilinear experiments we choose $\delta$, such that the sufficiently-bilinear condition is satisfied.
We show the results in Fig.~\ref{fig:spd-games}, we observe results very similar to the results observed in section~\ref{exp:bilinear-games} and section~\ref{exp:nonlinear-games}, the experiments again shows that our proposed methods follow closely the theory and that \textbf{L-SVRHG} is the fastest method to converge.

\begin{figure}[H]
\begin{subfigure}[t]{.5\textwidth}
   \begin{center}
   \centerline{\includegraphics[width=\columnwidth]{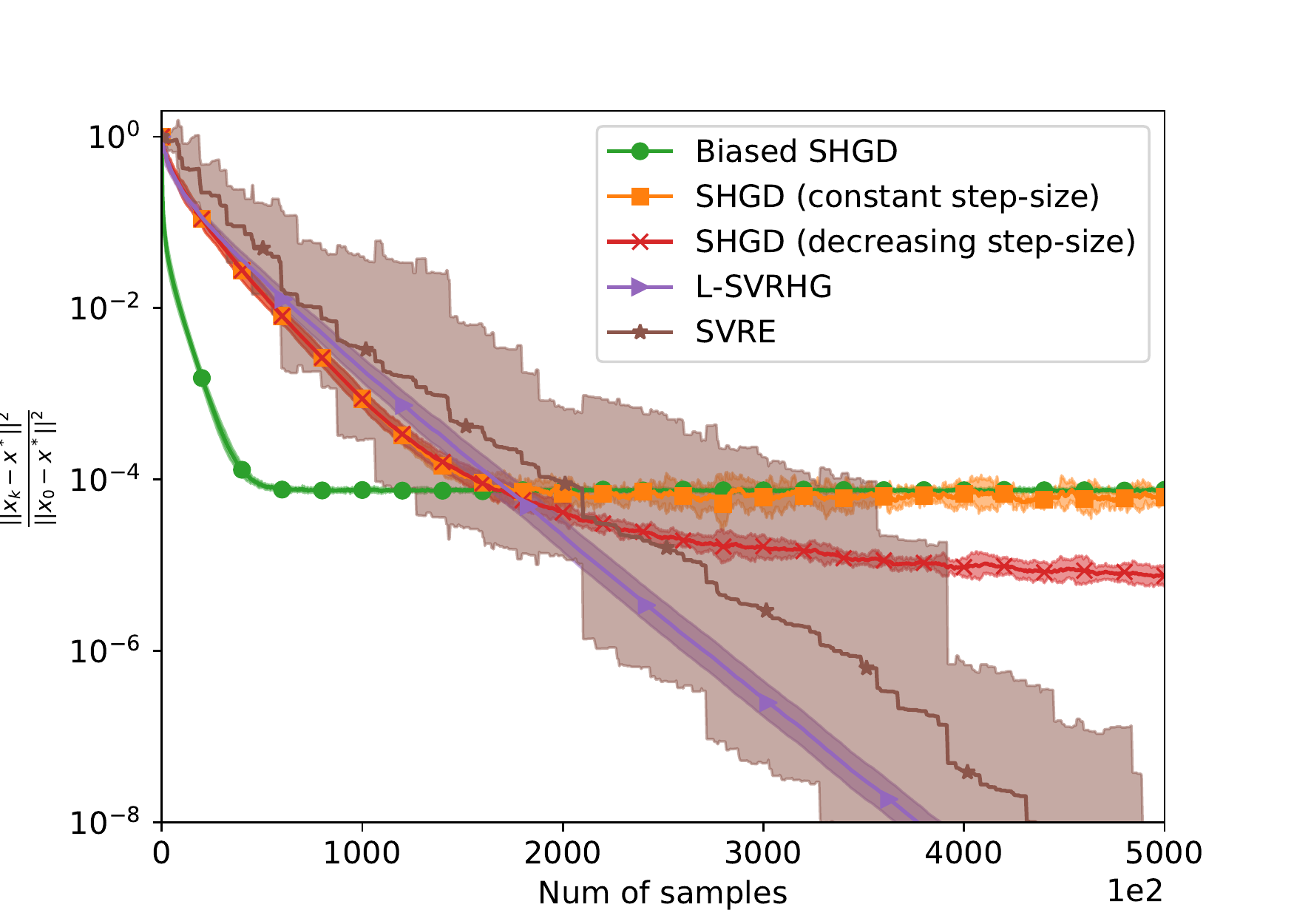}}
   \caption{Bilinear game}
   \label{fig:bilinear-game-spd}
   \end{center}
\end{subfigure}
\begin{subfigure}[t]{.5\textwidth}
   \begin{center}
   \centerline{\includegraphics[width=\columnwidth]{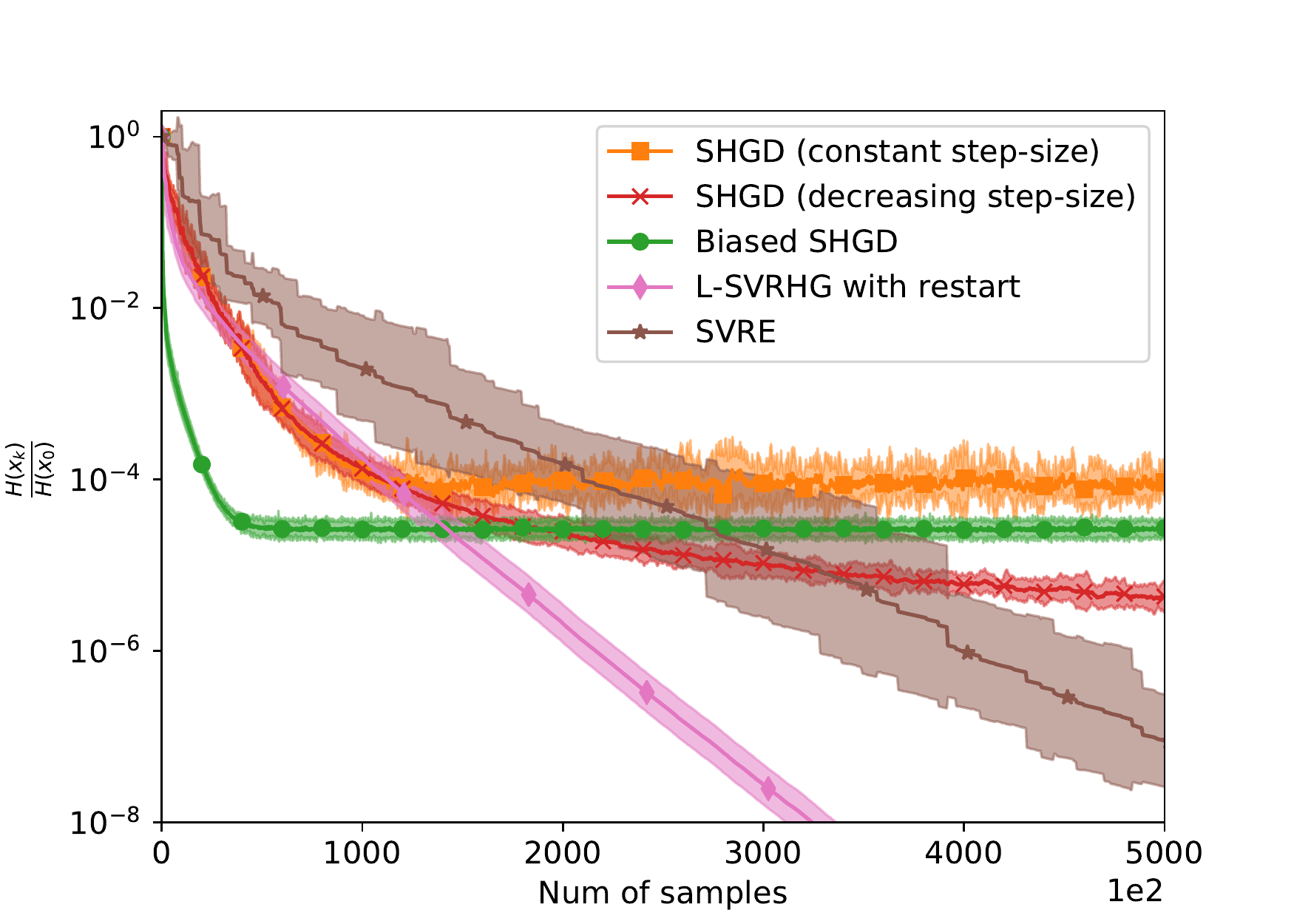}}
   \caption{Sufficiently-bilinear game}
   \label{fig:nonlinear-game-spd}
   \end{center}
\end{subfigure}
\vspace{-5mm}
\caption{Results for the bilinear game and sufficiently-bilinear game with symmetric positive definite matrices. We observe results very similar to the results observed in section~\ref{exp:bilinear-games} and section~\ref{exp:nonlinear-games}, the experiments again shows that our proposed methods follow closely the theory and that \textbf{L-SVRHG} is the fastest method to converge.}
\label{fig:spd-games}
\end{figure}

\subsubsection{Interpolated Games}

In this section we present a particular class of games that we call interpolated games.

\begin{definition}[Interpolated Games]
If a game is such that at the equilibrium $x^*$, we have $\forall i \; \xi_i(x^*)=0$, then we say that the game satisfies the interpolation condition.
\end{definition}

If a game satisfies the interpolation condition, then \textbf{SHGD} with constant step-size converges linearly to the solution.

In the bilinear game \eqref{bilinearGame1} and the sufficiently-bilinear game from Section~\ref{exp:nonlinear-games}, if we choose to set $\forall i \; b_i=c_i=0$, then both problems satisfies the interpolation condition.
We provide additional experiments in this particular setting where we compare \textbf{SHGD} with constant step-size, \textbf{Biased SHGD}, and \textbf{L-SVRHG}. We show the results in Fig.~\ref{fig:interpolated-games}.
We observe that all methods converge linearly to the solution, surprisingly in this setting \textbf{Biased SHGD} converges much faster than all other methods.

We argue that this is due to the fact that \textbf{Biased SHGD} is optimizing an upper-bound on the Hamiltonian.
Indeed we can show using Jensen's inequality, that:
\begin{equation}
\mathcal{H}(x)=\frac{1}{2}\|\xi(x)\|^{2}=\frac{1}{2}\|\frac{1}{n} \sum_{i=1}^{n} \xi_{i}(x)\|^{2} \stackrel{J e n s e n}{\leq} \frac{1}{2 n} \sum_{i=1}^{n}\left\|\xi_{i}(x)\right\|^{2}=\frac{1}{n} \sum_{i=1}^{n} \frac{1}{2}\left\|\xi_{i}(x)\right\|^{2}=\frac{1}{n} \sum_{i=1}^{n} \mathcal{H}_{i}(x)
\end{equation}

If the interpolation condition is satisfied, then we have that at the optimum $x^*$, the inequality becomes an equality:
\begin{equation}
\mathcal{H}(x^*)= \frac{1}{n} \sum_{i=1}^{n} \mathcal{H}_{i}(x^*) = 0
\end{equation}

Thus in this particular setting \textbf{Biased SHGD} also converges to the solution. 
Furthermore we can notice that because the $\bA_i$ are very sparse, $\forall i \neq j \;  \nabla\mathcal{H}_{i,j}(x)=0$. Thus most of the time \textbf{SHGD} will not update the current iterate, which is not the case of \textbf{Biased SHGD} which only considers the $\nabla\mathcal{H}_{i,i}(x)=0$ to do its update and thus always has signal.
The convergence of \textbf{SHGD} could thus be improved by using non-uniform sampling. We leave this for future work.

\begin{figure}[H]
\begin{subfigure}[t]{.5\textwidth}
   \begin{center}
   \centerline{\includegraphics[width=\columnwidth]{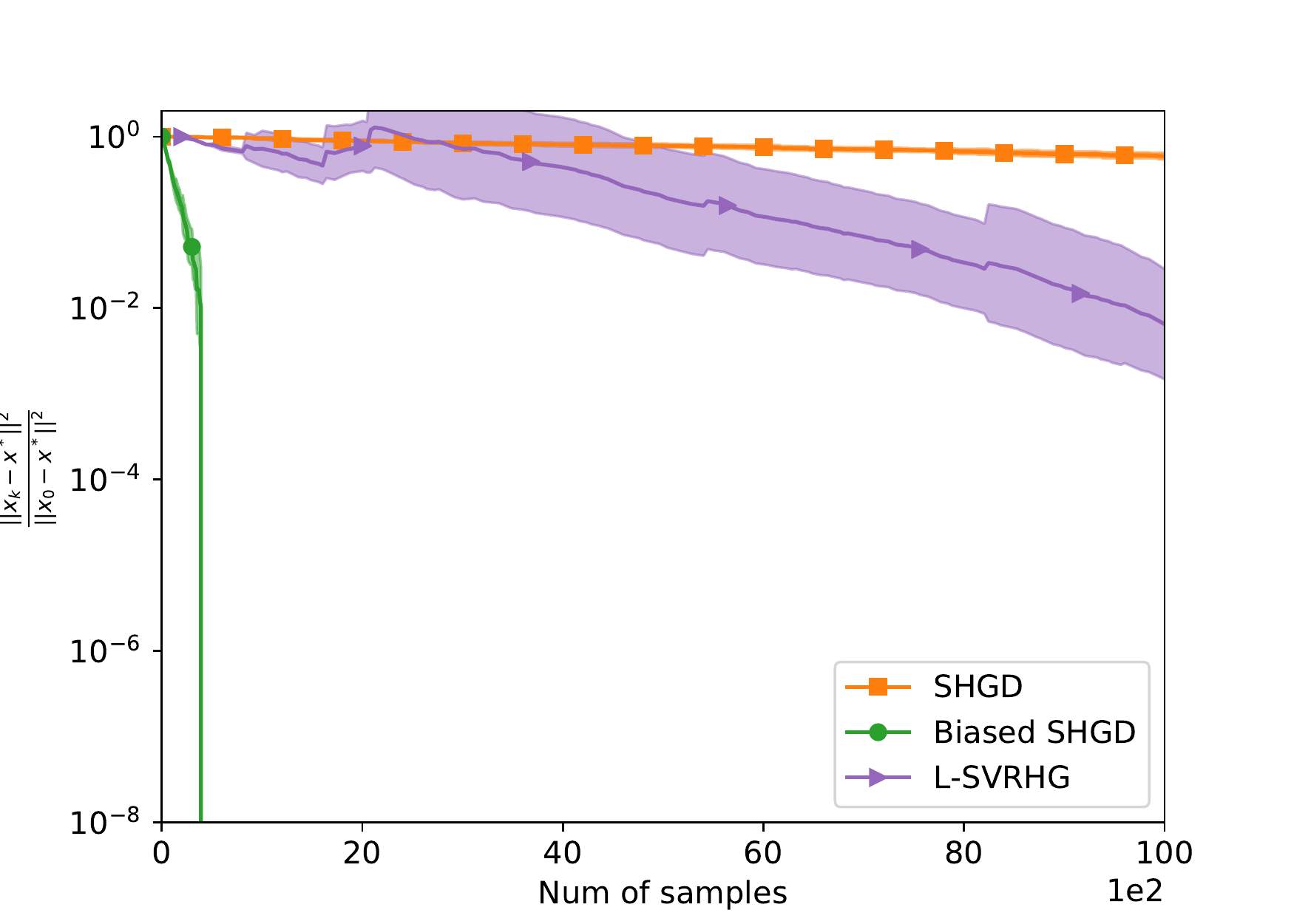}}
   \caption{Bilinear game}
   \label{fig:bilinear-game-interpolation}
   \end{center}
\end{subfigure}
\begin{subfigure}[t]{.5\textwidth}
   \begin{center}
   \centerline{\includegraphics[width=\columnwidth]{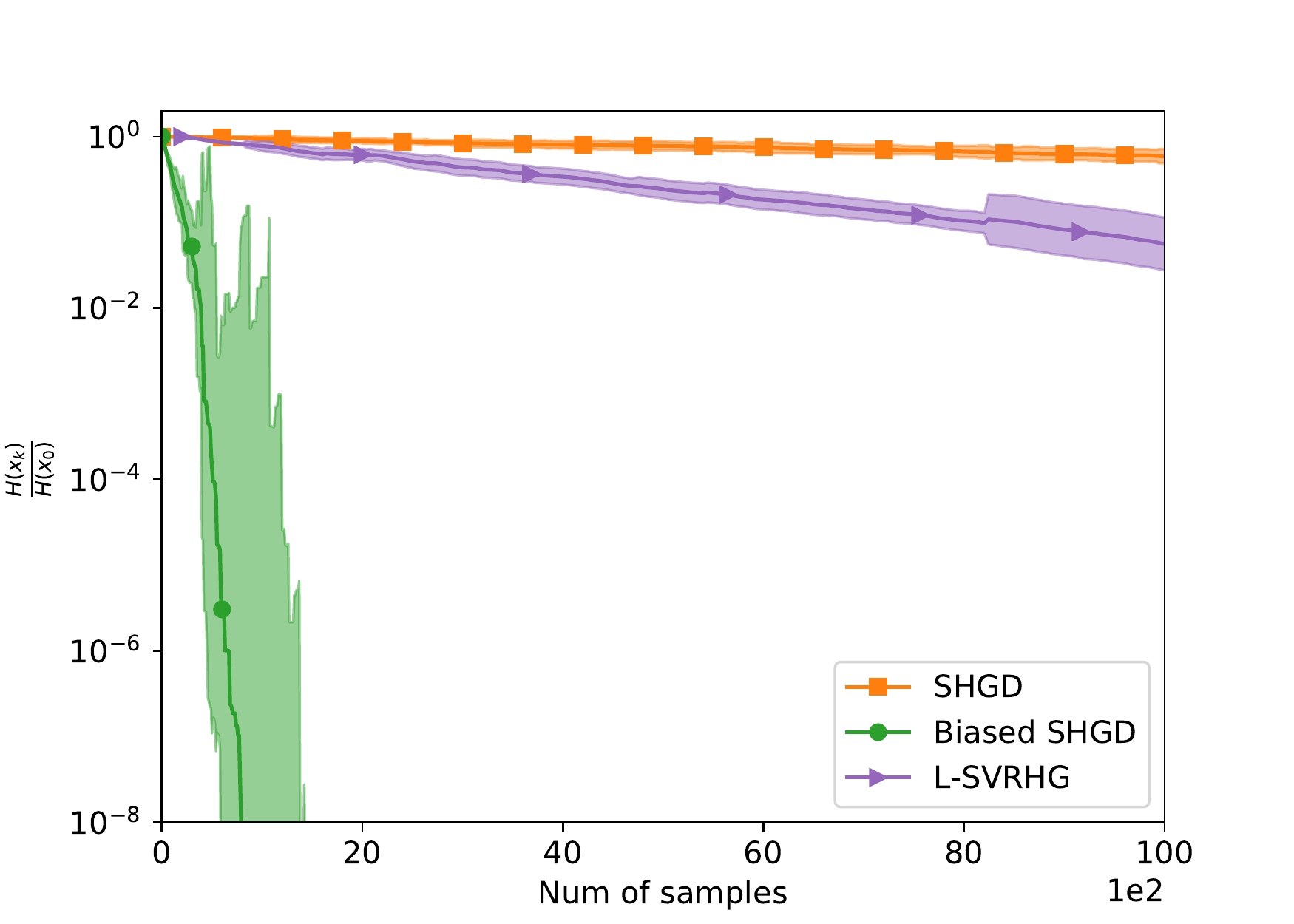}}
   \caption{Sufficiently-bilinear game}
   \label{fig:nonlinear-game-interpolation}
   \end{center}
\end{subfigure}
\vspace{-5mm}
\caption{Results for the bilinear game and sufficiently-bilinear game when $\forall i \; b_i=c_i=0$. We observe that all the methods converge linearly in this setting. Surprisingly in this setting \textbf{Biased SHGD} is the fastest method to converge. We give a brief informal explanation on why this is the case above.}
\label{fig:interpolated-games}
\end{figure}

\subsection{GANs}
\label{app:gans-other}

We present the missing experiments for satGAN (with batch size 100) in Figure~\ref{fig:gan1extra}. 
As can be observed, results for nsGAN are very similar to results for satGAN (see Figure~\ref{figgan1}).

\begin{figure}[H]	
	\centering	
	\begin{subfigure}[b]{0.475\textwidth}	
		\centering	
		\includegraphics[width=\textwidth]{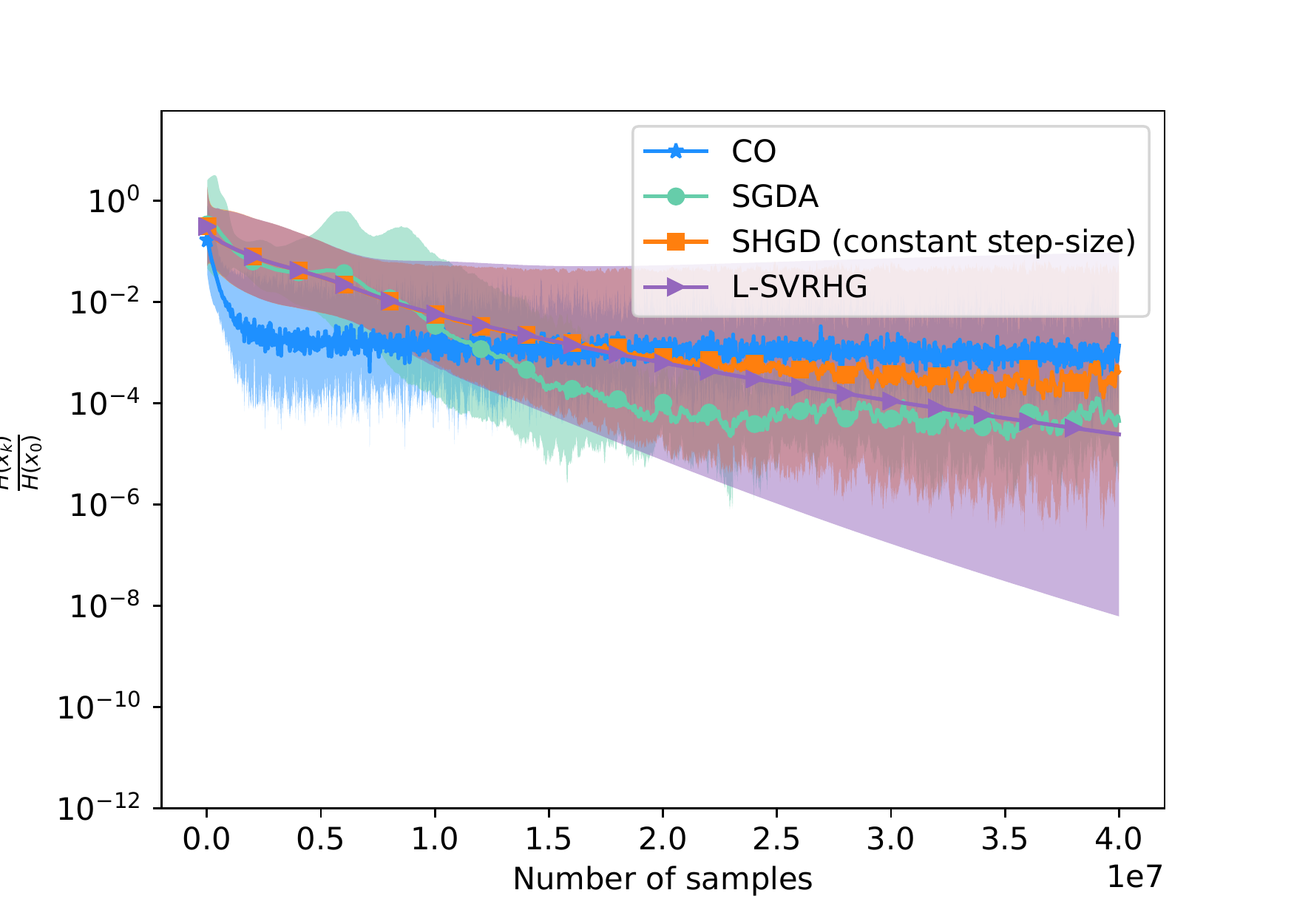}	
		\caption[Network2]%	
		{{\small Hamiltonian with nsGAN}}   	
	\end{subfigure}	
	\hfill	
	\begin{subfigure}[b]{0.475\textwidth}  	
		\centering 	
		\includegraphics[width=\textwidth]{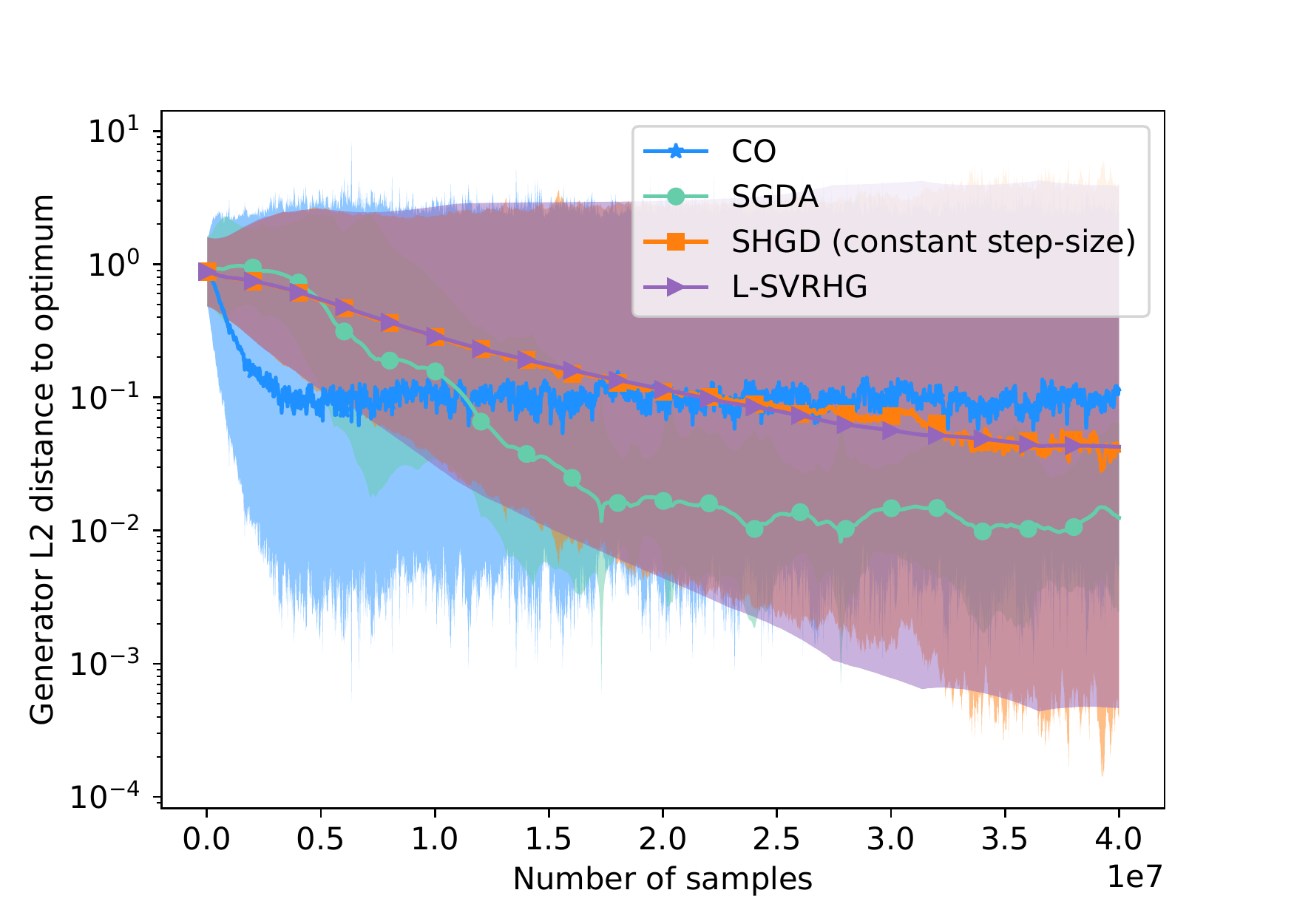}	
		\caption[]%	
		{{\small Distance to optimum with nsGAN}}    	
	\end{subfigure}	
	\caption{nsGAN with batch size of 100}	
	\label{fig:gan1extra}
\end{figure}	

\clearpage

\end{document}